\newif\ifarXiv         
\newif\ifjournal        

 \journalfalse \arXivtrue  

\ifjournal 
\documentclass[smallextended]{svjour3}       
\RequirePackage{fix-cm}
\smartqed  
\usepackage{graphicx}
\usepackage{lipsum}
\usepackage{amsfonts}
\usepackage{graphicx}
\usepackage{epstopdf}
\usepackage{algorithmic}

\usepackage[mathlines]{lineno}
\linenumbers

\usepackage{booktabs}
\usepackage{caption}
\usepackage{subcaption}
\usepackage[small,compact]{titlesec}
\usepackage{amsmath,amssymb,amsfonts,mathabx,setspace,bm} 

\usepackage{graphicx,caption,epsfig,epsfig,wrapfig}
\usepackage{url,color,verbatim,algorithmic}
\usepackage[sort,compress]{cite}
\usepackage[scaled]{helvet}

\usepackage[T1]{fontenc}

\usepackage{cite}
\usepackage{tikz}
\usetikzlibrary{positioning} 

\ifpdf
  \DeclareGraphicsExtensions{.eps,.pdf,.png,.jpg}
\else
  \DeclareGraphicsExtensions{.eps}
\fi

\title{Self-test loss functions for learning weak-form operators and gradient flows
 \thanks{Yuan Gao was partially supported by NSF under Award DMS-2204288. Fei Lu was partially supported by NSF DMS-2238486.}
}
\titlerunning{Self-test loss functions}        

\author{Yuan Gao \and Quanjun Lang \and Fei Lu}
\institute{Y. Gao\at Department of Mathematics, Purdue University, West Lafayette, IN, USA \\ \email{gao662@purdue.edu}
       \and Q. Lang\at Department of Mathematics, Duke University, Durham, NC, USA \\   \email{quanjun.lang@duke.edu} 
       \and F. Lu\at Department of Mathematics, Johns Hopkins University, Baltimore, MD, USA \\ \email{feilu@math.jhu.edu}
       }

\date{Received: date / Accepted: date}

\newcommand{\argmin}[1]{\underset{#1}{\operatorname{arg}\operatorname{min}}\;}

\def\R{\mathbb{R}}                                
\def\P{\mathbb{P}}

\def\barX{\overline X}
\def\E{ \mathbb{E}}
\def\calE{\mathcal{E}}        \def\calH{\mathcal{H}}

\newcommand{\innerp}[1]{\langle{#1}\rangle}

\def\spaceX{\mathbb{X}}
\def\spaceY{\mathbb{Y}}
\def\spaceM{\mathbb{M}}

\def\Gbar{{\overline{G}}}

\newcommand{\la}{\langle}
\newcommand{\ra}{\rangle}
\newcommand{\pt}{\partial}
\newcommand{\eps}{\varepsilon}
\newcommand{\ud}{\,\mathrm{d}}

\newcommand{\8}{\infty}

\newcommand{\bR}{\mathbb{R}}

\def\brho{\boldsymbol{\rho}}

 \newtheorem{assumption}[theorem]{Assumption}

\numberwithin{equation}{section}
\numberwithin{theorem}{section}

\fi

\ifarXiv     
\documentclass[12pt]{article}  

\usepackage[small,compact]{titlesec}
\usepackage{amsmath,amssymb,amsfonts,mathabx,setspace,bm} 
\usepackage{graphicx,caption,epsfig,epsfig,wrapfig}
\usepackage{url,color,verbatim,algorithmic}
\usepackage[sort,compress]{cite}
\usepackage[ruled,boxed]{algorithm}
\usepackage[scaled]{helvet}
\usepackage[T1]{fontenc}
\usepackage[bookmarks=true, bookmarksnumbered=true, colorlinks=true,   pdfstartview=FitV,
linkcolor=blue, citecolor=blue, urlcolor=blue]{hyperref}
\usepackage[sort,compress]{cite}
\usepackage{booktabs}
\usepackage{caption}
\usepackage{subcaption}
\usepackage{tikz}
\usetikzlibrary{positioning} 
\usepackage{xcolor}
\usepackage{authblk}

\newcommand{\argmin}[1]{\underset{#1}{\operatorname{arg}\operatorname{min}}\;}

\def\R{\mathbb{R}}                                
\def\P{\mathbb{P}}

\def\barX{\overline X}
\def\E{ \mathbb{E}}
\def\calE{\mathcal{E}}        \def\calH{\mathcal{H}}

\newcommand{\innerp}[1]{\langle{#1}\rangle}

\def\spaceX{\mathbb{X}}
\def\spaceY{\mathbb{Y}}
\def\spaceM{\mathbb{M}}

\def\Gbar{{\overline{G}}}

\newcommand{\la}{\langle}
\newcommand{\ra}{\rangle}
\newcommand{\pt}{\partial}
\newcommand{\eps}{\varepsilon}
\newcommand{\ud}{\,\mathrm{d}}

\newcommand{\8}{\infty}

\newcommand{\bR}{\mathbb{R}}

\def\brho{\boldsymbol{\rho}}

\newtheorem{theorem}{Theorem}

\newtheorem{assumption}[theorem]{Assumption}

\newtheorem{corollary}[theorem]{Corollary}
\newtheorem{definition}[theorem]{Definition}
\newtheorem{example}[theorem]{Example}

\newtheorem{proposition}[theorem]{Proposition}
\newtheorem{remark}[theorem]{Remark}
\newenvironment{proof}[1][Proof]{\noindent\textbf{#1.} }{\ \rule{0.5em}{0.5em}}

\numberwithin{equation}{section}
\numberwithin{theorem}{section}

\fi

\begin{document}


\ifarXiv
\title{Self-test loss functions
for learning weak-form operators and gradient flows}
\author[1]{Yuan Gao}
\affil[1]{{\small Department of Mathematics, Purdue University, West Lafayette,  USA} }
\author[2]{Quanjun Lang}
\affil[2]{{\small Department of Mathematics, Duke University, Durham, USA}} 
\author[3]{Fei Lu}
\affil[3]{{\small Department of Mathematics, Johns Hopkins University, Baltimore, USA}}
 \date{}
\fi
\maketitle


\begin{abstract} 
The construction of loss functions presents a major challenge in data-driven modeling involving weak-form operators in PDEs and gradient flows, particularly due to the need to select test functions appropriately. We address this challenge by introducing self-test loss functions, which employ test functions that depend on the unknown parameters, specifically for cases where the operator depends linearly on the unknowns. The proposed self-test loss function conserves energy for gradient flows and coincides with the expected log-likelihood ratio for stochastic differential equations. Importantly, it is quadratic, facilitating theoretical analysis of identifiability and well-posedness of the inverse problem, while also leading to efficient parametric or nonparametric regression algorithms. It is computationally simple, requiring only low-order derivatives or even being entirely derivative-free, and numerical experiments demonstrate its robustness against noisy and discrete data.
 \end{abstract}

\noindent \textbf{Keywords}: inverse problem; weak form; variational inference; gradient flow; operator learning    \\ 
\noindent \textbf{Mathematics Subject Classification (2020)}: 35R30,  68T09
\ifarXiv
 \vspace{-1mm}
\setcounter{secnumdepth}{2}
\setcounter{tocdepth}{1}
\tableofcontents
\fi 

\section{Introduction}
Learning governing equations from data is a fundamental task in many areas of science and engineering, such as physics, biology, and geosciences \cite{champion2019data,schaeffer2017learning,song2024towards,gilpin2020learning,owhadi2019kernel,hamzi2021learning,long2018_PDENetLearning}. The governing equation allows us to model complex systems, predict future behavior, and develop effective control strategies. They are often in the form of partial differential equations (PDEs), such as gradient flows \cite{JW16,carrillo2019aggregation,gao2019gradient,gao2020analysis} and diffusion models \cite{carrillo2024sparse,ghattas2021learning,luli2024learningEV,song21score,yang2022generative,isakov2006inverse}. To learn these equations, it is necessary to use data to approximate the differential operators. However, real-world data is often \emph{noisy and discrete}, leading to large errors in derivative approximations and unreliable estimators when using strong-form equations. 

Weak-form equations provide a more versatile framework. By using smooth test functions with integration by parts, weak forms use lower-order differential operators, thereby offering improved robustness to noisy and discrete data \cite{kharazmi2019variational,ghattas2021learning,messenger2021weak,messenger2021weakSindy,zang2020weak,stephany2024weak,gao2022physics}. 

However, constructing loss functions for variational inference of weak-form equations poses a major challenge. This difficulty arises because the weak form requires test functions to be dense in the dual space, typically an infinite-dimensional function space. In classical approaches, test functions are often chosen to be smooth and compactly supported, with Galerkin basis functions being a prominent example \cite{messenger2021weakSindy}. These methods are often limited to low-dimensional problems and are not scalable to high-dimensional settings, such as the Wasserstein gradient flows of probability measures in high-dimensional spaces. Importantly, since universal test functions are agnostic to the data and the model, it is necessary to use a large set of such test functions to ensure that all relevant information from the data is captured, which often leads to redundancy and computational inefficiency.  

We address this challenge by introducing \emph{self-test loss functions} for cases where the operator depends linearly on the  (function-valued) parameter. The key idea is to employ test functions that depend on the unknown parameter itself and the data, which we term \emph{self-testing functions}. Such test functions are automatically determined by the operator and the data. Thus, they automate the construction of the loss function. 

The proposed loss function is suitable for various weak-form operators, including the high-dimensional gradient flows and diffusion models. In particular, the selt-test loss function is quadratic. It facilitates theoretical analysis of identifiability and well-posedness of the inverse problem. It also enables efficient parametric and nonparametric regression algorithms. It is computationally simple, requiring only low-order derivatives or even being entirely derivative-free. Our numerical experiments demonstrate its robustness against noisy and discrete data.

\subsection{Problem settings and main results}\label{sec:exampleList}

Consider the problem of estimating the (function-valued) parameter $\phi$ in the operator $R_\phi:\spaceX \to \spaceY$ in the weak-form equation:  
\begin{equation}\label{eq:weak-operatorEq}	
	R_\phi[u] = f \quad   \Leftrightarrow  \quad
 \innerp{R_\phi[u],v} = \innerp{f,v},  \, \, \forall v\in \spaceY^*
\end{equation}
from data consisting of noisy discrete observations of input-output pairs: 
\begin{equation}\label{eq:data_uf}
\mathcal{D} = \{(u_l,f_l)\}_{l=1}^L. 
\end{equation}
Here, $\spaceX,\spaceY$ are metric spaces, $\spaceY^*$ is the dual space of $\spaceY$, and $\innerp{\cdot, \cdot}$ means the dual pair between $\spaceY$ and $\spaceY^*$. The operator $R_\phi: \spaceX\to \spaceY$ can be either linear or nonlinear. Depending on the operators, the data can be the functions at discrete spatial-time meshes or empirical distributions of samples; see \eqref{eq:data_uxt}, \eqref{eq:data_uf_disc} and \eqref{eq:data_Xtl} below.  

We assume that the operator $R_\phi[u]$ depends linearly on $\phi$ when $u$ is fixed, that is, 
      \begin{equation}\label{eq:R-linear}
      	R_{\alpha\phi_1+\beta \phi_2}[u]=\alpha R_{\phi_1}[u]+\beta R_{\phi_2}[u], 
      \end{equation} 
      for any $\alpha,\beta\in \R$ and any function $\phi_1$ and $\phi_2$ such that the operator is well-defined. 	
We assume no prior knowledge of $\phi$, except that the operator $R_\phi[u]$ is well-defined.

To construct a loss function using the weak form equation, we introduce \emph{self-testing functions} $v_\phi[u]\in\spaceY^*$ defined from $R_\phi$ and $u$ so that, for all $\phi,\psi$,
\begin{equation}\label{eq:core-props}
\langle R_\phi[u],\, v_\psi[u]\rangle  =  \langle R_\psi[u],\, v_\phi[u]\rangle \; \text{(symmetry)},\quad
\langle R_\phi[u],\, v_\phi[u]\rangle  \ge 0 \; \text{(positivity)}.
\end{equation}
The \emph{self-test loss function} is 
\begin{equation*}
    \mathcal{E}_\mathcal{D}(\phi) = \frac{1}{L}\sum_{l=1}^L \langle R_\phi[u_l], v_\phi[u_l] \rangle - 2 \langle f_l, v_\phi[u_l] \rangle + C_0, 
    \end{equation*}
    where $C_0$ is an arbitrary constant.  

{\color{black}    We demonstrate the self-test loss function in three settings involving function-valued parameters: Wasserstein gradient flows, a weak-form elliptic operator, and interacting particle systems with sequential ensembles of unlabeled data. Among these, Example~\ref{exmpl:GF} serves as a running example throughout the paper.
}

\begin{example}[Wasserstein gradient flow] \label{exmpl:GF}
Estimate $\phi=(h,\Phi,V)$ in the Wasserstein gradient flow 
	   \begin{equation}\label{eq:W2-GF}
		 \begin{aligned}
		  \partial_t u & =  \nabla\cdot(u \nabla [\nu h'(u) + \Phi*u + V] )=:R_\phi[u],  
		 \end{aligned}
	    \end{equation}
	    where $h:\R\to \R $ is the diffusion rate function, $\Phi:\R^d\to \R $ is the pairwise interaction potential satisfying $\Phi(-x)=\Phi(x)$, and $V:\R^d\to \R $ is an external potential acting on each particle.  
  The data consists of discrete noisy observations of solutions on a mesh $\{x_i\}_{i=1}^N\subset \R^d$:  
	\begin{equation}\label{eq:data_uxt}
		\mathcal{D}_1 =\{ u_l(x_i)\}_{i,l=1}^{N,L}, \quad u_l(x_i)= u(x_i,t_l)+ \epsilon_{i,l}, 
\end{equation} 
where $\{\epsilon_{i,l}\}$ are noises or measurement errors.  
{\color{black}The self-testing function is $v_\phi[u]=\nu h'(u) + \Phi*u + V $, and the self-test loss function is 
\begin{equation}\label{eq:stLoss-mf_all}
  \calE_u(\phi)   : =  \frac{1}{T}\int_0^T \int_{\R^d} \big[ u | \nabla [  \nu  h'(u) + \Phi * u + V ] |^2 - 2 \partial_t u [ \nu h'(u) +\Phi*u + V ]  \big ] dx dt, 
\end{equation} 
see Section {\rm \ref{sec:stloss-GF}} for a derivation and Section {\rm\ref{sec:ID}} for analysis on identifiability and well-posedness of the inverse problem. One can construct an empirical loss function by approximating the integrals in \eqref{eq:stLoss-mf_all} using data $\mathcal{D}_1$. }
\end{example}

\begin{example}[Weak-form operator]\label{exmpl:weakPDE}
 Estimate the coefficient $a:\R^d\to \R$ in the PDE:   
        \begin{equation}\label{eq:Lap-au}
       R_a[u] := - \Delta ( a u) = f  	
        \end{equation}
from data consisting of discrete noisy observations on the spatial mesh $\{x_i\}_{i=1}^N\subset \R^d$:   
\begin{equation}\label{eq:data_uf_disc}
\mathcal{D}_{2} =\{ (u_l(x_i), f_l(x_i)\}_{i,l=1}^{N,L}. 
\end{equation}
The self-testing function is $v_a[u] = au$; see Section {\rm\ref{sec:stloss-weakPDEs}}. Note that this inverse problem is different from the inverse conductivity problem (see e.g., {\rm\cite{isakov2006inverse}}), where the goal is to estimate $a$ in $\nabla\cdot (a\nabla u) = 0 \text{ in }\, \Omega $ when given only $u\big|_{\partial \Omega}=f $.    

\end{example}

\begin{example}[Sequential ensembles of unlabeled data]\label{exp:IPS} 
 Estimate the potentials $\Phi,V:\R^d\to \R $ in the differential equation of $N$-interacting particles, 
\begin{equation}\label{eq:IPS_deterministic}
   \frac{d}{dt}X_t^i  = -\big[ \nabla V(X_t^i) + \frac{1}{N}\sum_{j=1}^N \nabla \Phi(X_t^i-X_t^j)  \big] 
     ,\quad   	1\leq i\leq N 
\end{equation}
from data consisting of $M$ independent sequences of ensembles of \textbf{unlabeled} particles
   \begin{equation}\label{eq:data_Xtl} 
       \mathcal{D}_3 =\{( X_{t_l}^{i_l, (m)}, 1\leq i_l\leq N) \}_{m,l=1}^{M,L}. 
	\end{equation}
	 Since the particles are unlabeled, there is no information on their trajectories. Thus, the classical methods based on the derivatives $\frac{d}{dt} X_t^i$, see e.g., {\rm\cite{LZTM19pnas,LMT21_JMLR,LMT21}}, are no longer applicable. We construct a loss function based on the weak-form equation of the empirical measures $\mu_N(x,t):=\frac{1}{N}\sum_{i=1}^{N} \delta_{X^i_t}(x) $, and the self-testing function is $v_\phi[\mu_N]= \Phi * \mu_N + V $; see Section {\rm\ref{sec:ensembleData}}. Additionally, we demonstrate numerical estimation using neural network approximation in Section {\rm\ref{sec:num-PhiV}}.
\end{example} 

{\color{black}
\paragraph{Key features of the self-test loss function.} The quadratic self-test loss function conserves energy for gradient flows, aligns with the expected log-likelihood ratio for stochastic differential equations, facilitates theoretical analysis of identifiability and well-posedness, and leads to efficient parametric or nonparametric regression algorithms. 
\begin{itemize}
\item It aims to match the energy dissipation for the Wasserstein gradient flow, and its minimizer \emph{conserves the energy} of the data flow; see Theorem \ref{thm:GF_EnergyConservation}. The self-testing functions are the first variation of the free energy. Also, for the weak-form Fokker-Planck equation of the McKean-Vlasov stochastic differential equation (SDE), the self-test loss function coincides with the expectation of the negative log-likelihood ratio (see Theorem \ref{thm:stloss-likelihood}). As a result, a minimizer of the self-test loss function maximizes the expected likelihood. 
\item Importantly, the loss function is quadratic since both $R_\phi[u]$ and $v_\phi[u]$ are linear in $\phi$. It facilitates analysis on \emph{identifiability} and \emph{well-posedness} of the inverse problem based on the Hessian of the loss function. We demonstrate such an analysis for learning the diffusion rate function and the potentials in the Wasserstein gradient flow in Section \ref{sec:ID}. 
\item  It also leads to computationally efficient parametric or nonparametric regression algorithms, using either least-squares or neural network regression. We demonstrate its robustness against noisy and discrete data in parametric and nonparametric estimations in Section \ref{sec:numerics}.
\end{itemize}
}

\subsection{Related work} 
Weak formulations offer a robust and flexible foundation for addressing both forward and inverse PDE problems, and have thus attracted growing attention in recent years.  

{\color{black}
\noindent\textbf{General forward and inverse problems using weak-form.}  For forward problems, machine learning methods rooted in variational principles include the Deep Ritz method \cite{e2018_DeepRitz}, the Deep Galerkin method \cite{sirignano2018dgm}, variational physics-informed neural networks \cite{de2024wpinns,kharazmi2019variational,kharazmi2021hp}, and physics-informed graph neural Galerkin networks \cite{gao2022physics}, among others. For inverse PDE problems, we refer to \cite{isakov2006inverse} and \cite{ghattas2021learning} for comprehensive overviews. In classical inverse settings, where data are often limited to boundary measurements (e.g., in the inverse conductivity problem) or spectral information (e.g., in inverse spectral problems), one must estimate both the solution and the unknown parameters simultaneously. Recent approaches, such as weak adversarial networks \cite{bao2020numerical} and physics-informed graph neural Galerkin networks \cite{gao2022physics}, use weak form equations to address these classical difficulties. 

In contrast, our setting considers data consisting of PDE solutions sampled on discrete spatial grids or approximated by empirical measures, and the task is to estimate the PDE parameters. On the other hand, our self-test loss function can be applied to these methods, providing a systematic way to construct loss functions based on weak-form equations.  
}

\noindent\textbf{Regression based on weak-form.} Weak-form methods for parameter estimation have been widely explored in sparse regression frameworks, including Weak-SINDy \cite{messenger2021weak,messenger2021weakSindy,messenger2024weak}, Weak-PDE-LEARN \cite{stephany2024weak}, and other weak-form-based data-driven modeling approaches \cite{schaeffer2017learning,ghattas2021learning,song2024towards,carrillo2024sparse,ren2022data}. These methods rely on carefully designed families of smooth, compactly supported test functions that must be tailored to the data, domain, and PDE structure, a task that becomes increasingly challenging and computationally demanding in high dimensions. \textcolor{black}{In contrast, the self-test loss function proposed in this work removes the need for such hand-crafted test sets by constructing test functions directly from the operator and the data in a canonical way. Moreover, it can be seamlessly integrated into the Weak SINDy framework: one may include the self-test function as an additional test function or augment the Weak SINDy loss with the self-test term, thereby combining SINDy's sparsity-promoting structure with the robustness and adaptivity of the self-test formulation.}

\noindent\textbf{Energy variational approaches and gradient flow inference.} Our framework is closely related to energy variational approaches \cite{luli2024learningEV,wang2021particle,huliu24evNN} and gradient flow inference \cite{LangLu22,carrillo2024sparse}. The energy-dissipation-based loss \cite{luli2024learningEV,wang2021particle} shares conceptual similarities with the self-test loss function, aiming to preserve energy structures observed in the data. In particular, both approaches accommodate PDEs or stochastic differential equations for generalized diffusions and gradient flows, and can handle data defined on spatial grids or represented by particle samples. Furthermore, in the context of gradient flow inference, the self-test loss aligns with likelihood-based loss functions \cite{LangLu22} and the quadratic loss \cite{carrillo2024sparse}. By casting these methods into a unified variational inference framework, the self-test loss function extends their applicability beyond energy-dissipating systems to general weak formulations.

\bigskip
The rest of the paper is organized as follows. Section \ref{sec_self_test_loss} defines the framework of the self-test loss functions and provides examples. In Section \ref{sec:GF-energy}, we show that for general gradient flow, the self-test loss function's minimizers conserve the energy. We also connect it with the likelihood of SDEs. In Section \ref{sec:ID}, we study the identifiability and well-posedness of the diffusion rate function and the potentials in aggregation-diffusion equations. We present numerical experiments in Section \ref{sec:numerics} and conclude in Section \ref{sec:conclusion}.

\noindent\textbf{Notation.} Throughout the paper, we denote the true parameter by $ \phi_* $ and observational data by $ f $. We abuse the notation $u$, which may represent either a function $u(x)$ or $u(x,t)$ for a given $t$, as long as the context is clear.  Table \ref{tab:notation} lists the notations. 

 \begin{table}[htbp]\caption{Notations}
\centering
\begin{tabular}{c | l}
\toprule
 \textbf{Notations} & \textbf{Description}   \\
\midrule
$R_\phi[\cdot], \quad v_\phi[\cdot]$  &  Operators $R_\phi[\cdot]: \spaceX\to \spaceY, \quad v_\phi[\cdot]:\spaceX\to \spaceY^*$ \\
$\phi$ & (Function-valued) parameter to be estimated from data\\
$\innerp{f,v}, \, \quad \innerp{\cdot,\cdot}_H$  &  Dual operation with $f\in \spaceY, v\in \spaceY^*$; inner product in $H$ \\
$\mathcal{E}(\cdot)$ $E_\phi(\cdot)$  &  $\R$-valued loss function and energy function   \\
$\mathcal{P}_2(\R^d)$ & The space of probability measures with finite second moments\\  
\bottomrule
\end{tabular}
\label{tab:notation}
\end{table}


\section{Self-test loss functions} \label{sec_self_test_loss}

We first formulate the self-test loss function within a general \emph{weak-form operator learning} setting that includes both weak-form PDEs and gradient flows. The formulation is then illustrated through the running examples. For clarity, we derive the loss functions under the assumption of continuous noiseless data, before discussing how they are approximated from discrete, noisy measurements in practice. 

\subsection{Weak-form operator learning}\label{sec:stLoss_def}
The main idea behind the self-test loss function is to guide the minimization in the direction that explores the unknown parameter the most. Thus, we use the parameter to construct test functions.  

\begin{definition}[Self-test loss function]\label{def:self-test-loss} 
Consider the problems of estimating $\phi$ in the operator equation \eqref{eq:weak-operatorEq} from the dataset in \eqref{eq:data_uf}, where the operator $R_\phi[u]$ is linear in $\phi$. We call $v_\phi[u]\in \spaceY^*$ a \textbf{self-testing function} if it satisfies the self-testing properties: 
\begin{equation}	
	\begin{aligned}\label{def:self-testFn}
	\text{Symmetry:} &&  \quad \innerp{R_{\phi}[u],v_{\psi}[u]} & = \innerp{R_{\psi}[u],v_{\phi}[u]},  \\
	\text{Positivity:} &&  \quad \innerp{R_{\phi}[u],v_{\phi}[u]} & \geq 0, \\
	\text{Linearity:} &&  \quad  v_{\phi+\psi}[u] & = v_{\phi}[u] + v_{\psi}[u],  
    \end{aligned}
\end{equation}
for any $\phi,\psi$ such that these operations are well-defined for all $u\in \{u_l\}_{l=1}^L$. We call  
	\begin{equation}\label{eq:st-Loss0}
	\calE_{\mathcal{D}}(\phi) = \sum_{l=1}^L \innerp{R_{\phi}[u_l],v_{\phi}[u_l]} - 2\innerp{f_l, v_\phi[u_l]} + C_0
\end{equation}
a \textbf{self-test loss function}, where $C_0$ is an arbitrary constant. 
\end{definition}

The self-test loss function has three appealing properties. First, it is \emph{quadratic} in the unknown parameter $\phi$. Thus, it is convex, and its minimizers can be computed using the broad class of regression techniques. Also, the uniqueness of the minimizer can be established in a proper function space, as well as the well-posedness of the inverse problem; see Section \ref{sec:ID}. Second, it employs the weak form operator, which requires either a low-order derivative or no derivatives of $u$, thereby avoiding numerical errors when approximating derivatives from noisy discrete data. Lastly, in applications with probability gradient flow, it is particularly suitable for high-dimensional systems with ensemble data consisting of particle samples, as the loss function can be written as a combination of expectations; see Sections \ref{sec:ensembleData} and \ref{sec:likelihood}.

Two major tasks in the construction of the self-test loss function are (i) to find the self-testing function $v_{\phi}[u]$, and (ii) to select a proper parameter space for the minimization. Fortunately, the linearity of $R_\phi[u]$ in $\phi$ and the self-testing properties \eqref{def:self-testFn} provide clear clues on constructing $v_\phi[u]$. As examples, we explore such self-testing functions for weak-form PDEs and gradient flows in Sections \ref{sec:stloss-GF}--\ref{sec:stloss-weakPDEs}. Meanwhile, the loss function indicates adaptive function spaces for the parameter, which we explore in Section \ref{sec:ID}.

 The next proposition shows that any minimizer of the self-test loss function satisfies the weak-form equation when tested against all admissible self-testing functions. 
\begin{proposition}[Minimizer of the self-test loss function.]\label{prop:minimizer_self_test_loss}
Let $u_l$ be a weak solution to $R_{\phi_*}[u_l]=f_l$ for each $1\leq l\leq L$. 
The self-test loss function in \eqref{eq:st-Loss0} with $C_0= \sum_{l=1}^L\innerp{R_{\phi_*}[u_l],v_{\phi_*}[u_l]}$ can be written as 
	\begin{equation}\label{eq:st-Loss-square}
	\calE_{\mathcal{D}}(\phi) = \sum_{l=1}^L \innerp{R_{\phi-\phi_*}[u_l],v_{\phi-\phi_*}[u_l]} 
\end{equation}
and it has $\phi_*$ as a minimizer. In particular, $\phi_*$ is the unique minimizer in a linear space $\calH$ if and only if there exists $l\in \{1,\ldots, L\}$ such that $\innerp{R_{\phi}[u_l],v_{\phi}[u_l]}>0$ for every nonzero $\phi\in \calH$. Also, any minimizer $\phi_0$ of the self-test loss function is a solution to the equation
	\begin{equation}\label{eq:minimizerEq}
		\sum_{l=1}^L \innerp{R_{\phi_0}[u_l] - f_l,v_{\psi}[u_l]}=0, 
	\end{equation}
	for all $\psi$ such that $\sum_{l=1}^L \innerp{R_{\psi}[u_l] ,v_{\psi}[u_l]}<\infty$. 
\end{proposition}
\begin{proof} 
Given the above $C_0$, Eq.\eqref{eq:st-Loss-square} follows from  
	\begin{align*}
	\calE_{\mathcal{D}}(\phi) & = \sum_{l=1}^L \big[ \innerp{R_{\phi}[u_l],v_{\phi}[u_l]} - 2\innerp{f_l, v_\phi[u_l]} + \innerp{R_{\phi_*}[u_l],v_{\phi_*}[u_l]} \big] \\ 
	& =\sum_{l=1}^L \innerp{R_{\phi-\phi_*}[u_l],v_{\phi-\phi_*}[u_l]}, 
\end{align*}
where the last equality follows from the facts that $R_\phi[u]$ and $v_\phi[u]$ are linear in $\phi$, and that $ \innerp{R_{\phi}[u_l],v_{\phi^*}[u_l]}=\innerp{R_{\phi_*}[u_l],v_{\phi}[u_l]}=\innerp{f_l, v_\phi[u_l]}$. Then, $\phi_*$ is a minimizer by the positivity property. Also, this equation implies that the uniqueness of the minimizer in the linear space $\calH$ is equivalent to the strict positivity of $\frac{1}{L}\sum_{l=1}^L\innerp{R_{\phi}[u_l],v_{\phi}[u_l]}$ for every nonzero $\phi\in \calH$. Thus, $\phi_*$ is the unique minimizer in $\calH$ iff there exists $l\in \{1,\ldots, L\}$ such that $\innerp{R_{\phi}[u_l],v_{\phi}[u_l]}>0$ for every nonzero $\phi\in \calH$. 

Lastly, since $\phi_0$ is a minimizer of the loss function, we have, for any $\psi$ s.t.~$\calE_{\mathcal{D}}(\phi_0+\epsilon \psi)<\infty$,  
\[
0 = \frac{d}{d\epsilon } \calE_{\mathcal{D}}(\phi_0+\epsilon \psi) = \lim_{\epsilon\to 0} \frac{\calE_{\mathcal{D}}(\phi_0+\epsilon \psi)- \calE_{\mathcal{D}}(\phi_0)}{\epsilon} =  \sum_{l=1}^L\innerp{R_{\phi_0}[u_l] - f_l,v_{\psi}[u_l]}, 
\]
and it gives Eq.\eqref{eq:minimizerEq}. 
\end{proof}

\subsection{Example: Wasserstein gradient flow}\label{sec:stloss-GF}
We consider first the estimation of function-valued parameters in the Wasserstein gradient flow in \eqref{eq:W2-GF} from data in Example \ref{exmpl:GF}. 
Here the diffusion constant can be either $\nu>0$ or $\nu=0$, and the diffusion rate function $h:\R\to \R$ satisfies that $r\mapsto r^dh(r^{-d})$ is convex non-increasing. Examples of such $h$ include 
\begin{equation}\label{eq:h-powerFn}
h(s)=s \frac{1}{m-1}s^{m-1} =  
  \begin{cases}
      \frac{1}{m-1} s^{m} ,& m>1, \\
    s\log s,  &  m=1,
  \end{cases}
\end{equation}
where we use the convention $\frac{1}{m-1}\rho^{m-1} = \log \rho$ when $m=1$. In particular, when $m=1$, we have $h'(u)=1+\log u$ and $\nabla\cdot(u\nabla h'(u)) = \nabla\cdot [u \nabla  (1+ \log u) ] =  \Delta u$, and \eqref{eq:W2-GF} becomes 
\begin{equation}\label{eq:FPE-V-Phi}
\begin{aligned}
  \partial_t u & = \nu  \Delta u + \nabla \cdot(u\nabla [V+\Phi*u]),   \quad x\in \R^d, t>0. 
\end{aligned}
\end{equation}
This is the mean-field equation for the large $N$ limit of the interacting particle system, 
\begin{equation}\label{eq:IPS}
dX_t^i  = -\big[ \nabla V(X_t^i) + \frac{1}{N}\sum_{j=1}^N \nabla \Phi(X_t^i-X_t^j)  \big] dt + \sqrt{2\nu} dW_t^i,  \quad 	1\leq i\leq N,
\end{equation}
where $(W_t^i)_{1\leq i\leq N}$ are $\R^d$-valued independent Brownian motions, and $(X_0^i)_{1\leq i\leq N}$ are independent samples of distribution $u(\cdot,0)$; see e.g., \cite{JW16,JW17}.

\noindent\textbf{Self-test loss function for estimating $(h,\Phi, V)$.} The task is to estimate the parameter $\phi= (h,\Phi, V)$ in the operator $R_\phi[u]$ in \eqref{eq:W2-GF}. 
Its self-testing function is   
\begin{equation}\label{eq:stfn_MF}
	v_\phi[u]: = \nu h'(u) +\Phi*u + V. 
\end{equation}
It is direct to verify the self-testing properties in \eqref{def:self-testFn}: clearly, the symmetry and linearity hold; the positivity holds since by integration by parts, 
$\innerp{R_\phi[u],v_\phi[u]} = \int_{\R^d} u | \nabla [  \nu  h'(u) + \Phi * u + V)] |^2 dx\geq 0$,   
for all $\phi$ such that $\innerp{R_\phi[u],v_\phi[u]}$ is well-defined. 

Hence, the self-test loss function for data $( u(t,x): t\in [0,T], x\in \R^d)$ is \eqref{eq:stLoss-mf_all}. 
Its minimizer matches the energy dissipation of the gradient flow, which we explore in Section \ref{sec:GF-energy}.

\noindent\textbf{Self-test loss function for estimating $(\Phi, V)$.} Assume that  $\Phi(-x) = \Phi(x)$. Consider the problem of estimating $(\Phi, V)$ in the mean-field equation \eqref{eq:FPE-V-Phi}, i.e., estimating the parameter $\phi= (\Phi, V)$ in the (weak-form) operator $R_\phi[u]= - \nabla\cdot [ u  \nabla ( \Phi * u + V)]$. The self-testing function is $v_\phi[u]= \Phi * u + V $, and $\innerp{R_\phi[u],v_{\phi}[u]} = \int_{\R^d}  u | \nabla \Phi * u + \nabla V |^2 dx $. Thus, the self-test loss function is 
\begin{align}\label{eq:stLoss-mf}
  \calE_u(\Phi,V)   & =  \frac{1}{T}\int_0^T \int_{\R^d} \left[ u | \nabla \Phi * u + \nabla V |^2 - 2 (\partial_t u-   \nu  \Delta u) (\Phi*u + V )  \right ] dx\ dt. \notag\\
  & = \frac{1}{T}\int_0^T \int_{\R^d} \left[ u | \nabla \Phi * u + \nabla V |^2 +  2 \nu u(\Delta \Phi*u + \Delta V) \right ] dx\ dt \notag \\
  & - \frac 2 T \int_{\R^d} \big[ u(T,x) [ \Phi*u(T,x)/2 + V(x) ] -  u(0,x) [ \Phi*u(0,x)/2 + V(x) ]\big] dx, 
\end{align} 
where the last equality follows from integration by parts and $\Phi(-x)= \Phi(x)$. 

In practice, when the data is discrete, as in \eqref{eq:data_uxt}, we approximate the integrals in \eqref{eq:stLoss-mf} using numerical methods, such as Riemann sums.

\subsection{Example: elliptic diffusion operators}\label{sec:stloss-weakPDEs}
 To estimate $a:\R^d\to \R$ in Example \ref{exmpl:weakPDE}, we have $R_a[u] = - \Delta ( a u): C_c^1(\R^d)\to \spaceY$.  Here $\spaceY$ is a Banach space such that $\text{BV}^* \subset \spaceY$ and $\spaceY^* \subset \text{BV}$, where $\text{BV}$ denotes the space of functions with bounded variation. The self-testing function is $v_a[u] = au\in C_c^1(\R^d)$, whose self-testing properties follow directly, in particular, $\innerp{R_a[u],v_a[u]} =- \int_{\R^d}  \Delta ( a u) au dx =  \int_{\R^d} |\nabla (au)|^2 dx\geq 0$ for all $a\in C_c^1(\R^d)$. 
 Hence, the self-test loss function for a single data pair $(u,f)$ is      
        \begin{align}\label{eq:loss_elliptic}
        	\calE_{(u,f)}(a) = \innerp{R_a[u]-2f,v_a[u]}  =  \int_{\R^d}  [ |\nabla (au)|^2 - 2 f au ] dx.  
        \end{align}
       Approximating the integrals by Riemann sums with the data in \eqref{eq:data_uf_disc}, we obtain an empirical self-test loss function 
         \begin{align*}
        	\calE_{\mathcal{D}_{2}} (a) &= \frac{1}{L}\sum_{l=1}^{L} \calE_{(u_l,f_l)}(a)     	  = \frac{1}{NL}\sum_{i,l=1}^{N,L} \left[ |[\nabla (au_l)](x_i))|^2 - 2 f_l(x_i) a(x_i)u_l(x_i)  \right] |\Delta x_i|.  
        \end{align*}

\subsection{Example: sequential ensembles of unlabeled data}\label{sec:ensembleData} 
To estimate the potentials from sequential ensembles of unlabeled data $( X_{t_l}^{i_l, (m)}, 1\leq i_l\leq N)$ in Example \ref{exp:IPS}, we consider the empirical measures of the data 
\[
\mu_N^{(m)}(x,t_l)= \frac{1}{N}\sum_{i_l=1}^N \delta_{X_{t_l}^{i_l, (m)}}(x).
\]

We construct a self-test loss function using the fact that the empirical measure $\mu_N(x,t):=\frac{1}{N}\sum_{i=1}^{N} \delta_{X^i_t}(x) $ with $(X_t^i,1\leq i\leq N)$ satisfying \eqref{eq:IPS_deterministic} is a weak solution to equation     
    \begin{equation} \label{eq:muN_GF}
      \partial_t \mu_N = \nabla \cdot(\mu_N \nabla [V+\Phi*\mu_N]),   \quad \mu_N(\cdot,t)\in \mathcal{P}_2(\R^d), \ t>0. 
\end{equation}
In other words, for any function $v\in C^2(\R^d)$, 
    \begin{align*}
        \innerp{\partial_t \mu_N, v} =\innerp{\nabla\cdot \big(\mu_N( \nabla \Phi * \mu_N + V) \big),  v} =  - \,\innerp{\mu_N( \nabla \Phi * \mu_N + V), \nabla v}, 
    \end{align*}
    where the second equality follows from integration by parts. In fact, the above equation holds by the chain rule with the differential equation \eqref{eq:IPS_deterministic}:    \begin{align*}
        \innerp{\partial_t \mu_N, v}& = \frac{1}{N}\sum_{i = 1}^N  \frac{d}{dt} v(X_t^i) = \frac{1}{N}\sum_{i = 1}^N  \frac{dX_t^i}{dt} \cdot  \nabla v(X_t^i) \\
                                     & = -\frac{1}{N}\sum_{i = 1}^N\left(  \frac{1}{N}\sum_{j = 1}^N \nabla \Phi(X_t^i - X_t^j) + \nabla V(X_t^i)\right)\cdot \nabla v(X_t^i) \\
                                    & =  -\innerp{\mu_N( \nabla \Phi * \mu_N + V), \nabla v}
    \end{align*}
and by noticing that $\nabla \Phi * \mu_N(x) = \frac{1}{N}\sum_{j = 1}^N \nabla \Phi(x - X_t^j)$. 

Thus, we consider the weak-form operator $R_\phi[u]= - \nabla\cdot [ u  \nabla ( \Phi * u + V)]$ with output $f = \partial_t u$. The self-testing function is $v_\phi[u]= \Phi * u + V $, and $\innerp{R_\phi[u],v_{\phi}[u]} = \int_{\R^d}  u | \nabla \Phi * u + \nabla V |^2 dx $. Then, the self-test loss function is \eqref{eq:stLoss-mf} with $\nu=0$. Using the data-induced empirical measures $\{ \mu_N^{(m)}(\cdot,t_l)\}$, we have a self-test loss function 
\begin{align}\label{eq:loss_MC_deterministic}
  \calE_{\mathcal{D}_3}(\Phi, V) &=  \frac{1}{LMN} \sum_{l,i,m=1}^{L,N,M} 
        \big|  \frac{1}{N}\sum_{j = 1}^N    \nabla \Phi(X_{t_l}^{i, (m)} - X_{t_l}^{j, (m)}) + \nabla V(X_{t_l}^{i, (m)}) \big|^2 dt \notag\\
  &-   \frac{2}{LMN} \sum_{i,m= 1}^{N,M}
       \left. \big[\frac{1}{N}\sum_{j = 1}^N\Phi(X_{t}^{i,(m)} - X_t^{j, (m)}) + V(X_t^{i, (m)}) \big]\right|^{t_L}_{t_1}. 
\end{align} 
Note that this empirical loss function does not use the trajectory information of any single particle, and it uses exactly the ensemble data of unlabeled particles. We demonstrate the application of this loss function in Section \ref{sec:num-PhiV}. 

\begin{remark}
Eq.\eqref{eq:muN_GF} is the same as \eqref{eq:W2-GF} with $\nu=0$ and the empirical measures $(\mu_N(\cdot,t), t\geq 0)$ form a Wasserstein gradient flow on $\mathcal{P}_2(\R^d)$. However, it is not the Liouville equation of the ODE in \eqref{eq:IPS_deterministic}, since the Liouville equation governs the evolution of the joint distribution on $\R^{Nd}$. Similarly, the mean-field equation in \eqref{eq:FPE-V-Phi} is not the Fokker-Planck equation of the SDE in \eqref{eq:IPS}, but we can use it to derive the same self-test loss function for the SDE with sequential ensembles of unlabeled data $\mathcal{D}_3$.        	
\end{remark}

\section{Connection with energy conservation and likelihood}
\label{sec:GF-energy}

This section connects the self-test loss to two fundamental principles: the energy conservation law of gradient flows and the maximal likelihood principle for inference in stochastic differential equations (SDEs). We show that the self-test loss is designed to match the energy dissipation of a gradient flow, and that its minimizer satisfies the corresponding \emph{energy conservation law} for the observed data. Moreover, the first variation of the free energy naturally yields a self-testing function. These results are illustrated through the Wasserstein and parabolic gradient flow examples. Finally, we show that, for SDEs, the self-test loss coincides with the expected negative log-likelihood ratio.

\subsection{Matching energy dissipation for gradient flow}
\label{sec:generalGF}

We first define the self-test loss function for a generic gradient flow whose free energy depends linearly on the parameter. 

Consider the estimation of the function-valued parameter $\phi$ in the free energy $E_\phi:\spaceM\to \R$, where $\spaceM$ is a metric space, from a gradient flow path $u_{[0,T]}: = (u(t,\cdot), t\in [0,T]) \subset \spaceM$. Here, the gradient flow satisfies the equation 
\begin{equation}\label{eq:gf-abstract}
    \pt_t u = - A_u  \frac{\delta E_\phi}{\delta u}, 
\end{equation}
where $\pt_t u \in T_u \spaceM$,  $A_u: T_u^* \spaceM \to T_u \spaceM$ is a nonnegative definite operator from the cotangent plane $T_u^* \spaceM$ to the tangent plane $T_u \spaceM$, and $\frac{\delta E_\phi}{\delta u}\in T_u^* \spaceM$  is the Fr\'echet derivative (also called the first variation) of the free energy. Its weak form reads 
\begin{equation*}
     \la \pt_t u, g \ra  + \la A_u \frac{\delta E_{\phi}}{\delta u}, g \ra=0, \quad \forall g\in T_u^* \spaceM,
\end{equation*}
where $\la \cdot, \cdot \ra$ is the dual pair on $T_u\spaceM\times T^*_u \spaceM.$

We define a self-test loss function for estimating $\phi$ by connecting the gradient flow with the weak form operator $R_\phi$ in \eqref{eq:weak-operatorEq} and its self-testing function $v_\phi[u]$ as follows: 
\begin{equation}\label{test1}
 R_\phi[u]=   A_u  \frac{\delta E_{\phi}}{\delta u}, \quad v_{\phi}[u] = \frac{\delta E_\phi}{\delta u}.   
\end{equation}
The following assumptions on the gradient flow ensure the self-testing properties in \eqref{def:self-testFn}.   
\begin{assumption}\label{assum:GF-abstract}
Assume the gradient flow in \eqref{eq:gf-abstract} satisfies the following properties. 
\begin{itemize}
\item[(i)] The operator $A_u$ is linear, nonnegative definite, and symmetric: $\forall \xi,\eta \in T_u^* \spaceM$,  
\begin{equation}\label{eq:Au-properties}
\begin{aligned}
     & \text{linear:}        & A_u (\xi+\eta) & = A_u \xi + A_u \eta ; \\
    & \text{symmetric:}   & \la A_u \xi, \eta \ra   & =  \la \xi, A_u \eta \ra; \\
   & \text{nonnegative definite:}    & \la A_u \xi, \xi \ra &\geq 0.  
\end{aligned}
\end{equation}
Here $\la \cdot, \cdot \ra$ are dual pair on $T_u \spaceM\times T^*_u \spaceM.$
\item[(ii)] The free energy $E_\phi$ depends on $\phi$ linearly. Consequently, $\frac{\delta E_{\phi}}{\delta u}$ is also linear in $\phi$, i.e.,
$    \frac{\delta E_{\phi+ \psi}}{\delta u} = \frac{\delta E_{\phi}}{\delta u}+ \frac{\delta E_{\psi}}{\delta u}  
 $ for all $\phi,\psi$ such that the energy function is well-defined. 
\end{itemize}
\end{assumption}

\begin{definition}[Self-test loss function for gradient flow]
\label{def:GFabstract-stLoss}
Consider the problem of estimating $\phi$ in the gradient flow \eqref{eq:gf-abstract} satisfying Assumption {\rm \ref{assum:GF-abstract}}.  
Given continuous time data $u_{[0,T]}: = (u(t,\cdot), t\in [0,T])$, a self-test loss function is 
\begin{equation}\label{eq:calE-cont-time}
   \calE_{u_{[0,T]}}(\phi) = 2 [ E_\phi(u(T,\cdot)) - E_\phi(u(0,\cdot))] + \int_0^T \la A_u \frac{\delta E_{\phi}}{\delta u}, \frac{\delta E_{\phi}}{\delta u} \ra dt . 
   \end{equation}
\end{definition}

The next theorem shows that the self-test loss function has the true parameter $\phi_*$ as a minimizer, and that its minimizer satisfies energy conservation for the data flow.   
We postpone its proof to Appendix \ref{App:GF-liklihood}.

 \begin{theorem}[Minimizer of the loss function]\label{thm:GF_EnergyConservation} 
 The minimizer of the loss function $\calE_{u_{[0,T]}}(\phi)$ in \eqref{eq:calE-cont-time} of Definition {\rm \ref{def:GFabstract-stLoss}} satisfies the following properties. 
 \begin{itemize}
  \item[(a)] The true parameter $\phi_*$ is a minimizer and $\calE_{u_{[0,T]}}(\phi_*)= -\int_0^T\innerp{A_u \frac{\delta E_{\phi_*}}{\delta u}, \frac{\delta E_{\phi_*}}{\delta u}} \ud t$. 
 \item[(b)] \textbf{Uniquenss.} The minimizer is unique in a linear parameter space  $\calH$ if  
\begin{equation}\label{eq:Au-positive}
	\int_0^T \la A_u \frac{\delta E_{\phi}}{\delta u}, \frac{\delta E_{\phi}}{\delta u} \ra\, dt >0, \quad  \forall \phi\in \calH,  \phi\neq 0. 
\end{equation}
 	\item[(c)] \textbf{Energy conservation.} A minimizer $\phi_0$ of $\calE_{u_{[0,T]}}(\phi)$ satisfies the energy conservation for the data $u_{[0,T]}$. That is, the energy change $E_{\phi_0} [u(T,\cdot) ] - E_{\phi_0} [u(0,\cdot) ] $ matches the total energy dissipation $-\int_0^T \la A_u \frac{\delta E_{\phi_0}}{\delta u}, \frac{\delta E_{\phi_0}}{\delta u} \ra \, dt$ along the flow $u_{[0,T]}$: 
\begin{equation}\label{eq:enerby-conservation}
	E_{\phi_0} [u(T,\cdot) ] - E_{\phi_0} [u(0,\cdot) ] = -\int_0^T \la A_u \frac{\delta E_{\phi_0}}{\delta u}, \frac{\delta E_{\phi_0}}{\delta u} \ra \, dt. 
\end{equation} 
 \end{itemize}
 \end{theorem}

\noindent\textbf{Example: the Wasserstein gradient flow.}
We show first that the Wasserstein gradient flow in Eq.\eqref{eq:W2-GF} satisfies Assumption \ref{assum:GF-abstract}; thus, its self-test loss function in \eqref{sec:stloss-GF} aims to match the energy dissipation in the data.

Let $\spaceM:=(\mathcal{P}_2(\R^d), W_2)$ be the space of probability measures with finite second moments endowed with the Wasserstein-2 metric $W_2$. Recall that for any convex functional $E(u)$ over  $\spaceM$, the gradient is 
$\nabla^{W_2} E(u) =  -\nabla\cdot (u \nabla \frac{\delta E}{\delta u})$, 
where $\nabla$ is the gradient with respect to $x$ (see e.g., \cite{villani2003topics, carrillo2019aggregation}). 
Then, a gradient flow in $\spaceM$ is  
\begin{equation}\label{eq:grad-flow}
  \partial_t u =- \nabla^{W_2}  E= \nabla\cdot (u \nabla \frac{\delta E}{\delta u}) = - A_u \frac{\delta E}{\delta u}\, \text{ with } A_u \xi := -\nabla\cdot (u \nabla \xi). 
\end{equation}
Clearly, the operator $A_u: T_u^*M \to T_u \spaceM$ is linear, non-negative definite and symmetric, i.e., it satisfies Assumption \ref{assum:GF-abstract}(i).

To connect with Eq.\eqref{eq:W2-GF}, consider the free energy with parameter $\phi= (h,\Phi,V)$: 
\begin{align*} 
E_{\phi}(u)=\nu \int h(u)+ \frac{1}{2} \int\int \Phi(x-y)u(x)u(y)dxdy  + \int V(x) u(x) dx. 
\end{align*} 
Here, the first term is called entropy (named when $h(s)= s\log s$) or internal energy in general, and the second and third terms are called interaction energy and potential energy. Since $\Phi(x) =\Phi(-x)$, the Fr\'echet derivative of this energy function is  
\begin{equation}\label{eq:deltaE}
\frac{\delta E_\phi}{\delta u}= \nu  h'(u)+  \Phi*u  + V.
\end{equation} 
Then, the $W_2$-gradient flow equation \eqref{eq:grad-flow} becomes Eq.\eqref{eq:W2-GF}. 

In particular, note that both $E_\phi$ and its derivative $\frac{\delta E_\phi}{\delta u}$ in \eqref{eq:deltaE} are linear in $\phi$. In other words, Assumption \ref{assum:GF-abstract}(ii) holds. Thus, we can define the self-test loss function in \eqref{eq:calE-cont-time}. Meanwhile, note that the above $\frac{\delta E_\phi}{\delta u}$ is exactly the self-testing function $v_\phi[u]$ in \eqref{eq:stfn_MF}. Thus, this self-test loss function agrees with the one in \eqref{eq:stLoss-mf_all}. 

Thus, by Theorem \ref{thm:GF_EnergyConservation}, the self-test loss function has $\phi_*$ as a minimizer, and any of its minimizers matches the energy conservation for the data flow.

\noindent\textbf{Example: the parabolic gradient flow}.  
Consider next estimating  the coefficient $a(x)$ from data $u_{[0,T]}$ of the parabolic (or $H^{-1}$) gradient flow 
\begin{equation}\label{heat1} 
    \partial_t u = \Delta (a(x)   u), \quad x\in \mathbb{T}^d, 
\end{equation}
where $\mathbb{T}^d$ is the $d$-dimensional torus.  
It is a $H^{-1}$ gradient flow of the free energy $ E_{a}(u):= \frac12 \int a(x) u^2 \ud x$ since $ \nabla^{H^{-1}} E = -\Delta \frac{\delta E_a}{\delta u} $ and $\frac{\delta E_a}{\delta u} = a u$.  
In other words, Eq.\eqref{heat1} can be written as  
\begin{equation*}
    \pt_t u = - \nabla^{H^{-1}} E 
    = - A_u \frac{\delta E_a}{\delta u} \, \text { with }     A_u \xi = -\Delta \xi.  
\end{equation*}
Clearly, Assumption \ref{assum:GF-abstract} holds since (i) the $A_u: H^1 \to H^{-1}$ is linear, nonnegative definite and symmetric, and (ii) the energy function $E_{a}$ and its derivative $\frac{\delta E_a}{\delta u}$ are linear in $a$. Thus, by \eqref{eq:calE-cont-time} with integration by parts, the self-loss function is 
\[
\calE_{u_{[0,T]}} (a) =  \int_{\mathbb{T}^d} [u(T,x))^2 - u(0,x))^2]a(x)dx +   \int_0^T \int_{\mathbb{T}^d} |\nabla (au)|^2 \ud x \ud t .  
\]
It is the time-integrated version of the loss function \eqref{eq:loss_elliptic} with $f=\partial_t u$ for Example \ref{exmpl:weakPDE}. 

\subsection{Expected likelihood ratio of the McKean-Vlasov SDE}\label{sec:likelihood}
Next, we show that for the McKean-Vlasov SDE, the self-test loss function of its Fokker-Planck equation coincides with the expectation of the negative log-likelihood ratio (see Appendix \ref{App:GF-liklihood} for its proof). 
\begin{theorem}\label{thm:stloss-likelihood}
Consider the problem of estimating the potentials $V_{*},\Phi_{*}:\R^d\to \R$ in the McKean-Vlasov SDE 
	\begin{equation} \label{eq:McKean-Vlasov}
\left\{
\begin{aligned}
d\barX_t =& -\nabla [V_{*}(\barX_t)+ \Phi_{*} *u(\barX_t,t) ]dt + \sqrt{2\nu }dB_t,\\
u(x,t) = &\E[\delta_{\barX_t}(x)]. 
\end{aligned}
\right.
\end{equation}
Suppose that the data is $u_{[0,T]}:= (u(t,x), t\in [0,T], x\in \R^d)$, where $u(t,\cdot)$ the probability distribution of $\barX_t$. Then, the self-test loss function  in  \eqref{eq:stLoss-mf}  for the weak form Fokker-Planck equation in \eqref{eq:FPE-V-Phi} is the expectation of the negative log-likelihood ratio $\mathcal{E}_{\barX_{[0,T]}}(\Phi,V)$ of the path $\barX_{[0,T]}$, i.e., $\calE_{u_{[0,T]}}(\Phi,V) = \frac{2\nu }{T }	 \E \big[ \mathcal{E}_{\barX_{[0,T]}}(\Phi,V) \big]$. 
\end{theorem}

A key advantage of the self-test loss function is its applicability for both $\nu > 0$ and $\nu = 0$. In contrast, the likelihood-based approach requires $\nu > 0$, as this condition is essential for applying the Girsanov theorem to define a non-degenerate measure on the path space. The self-test loss function, however, imposes no such constraint on $\nu$. Notably, when $\nu = 0$, the SDE reduces to an ordinary differential equation (ODE). When the ODE has a random initial condition, the self-test loss function is derived from the Liouville equation governing the distribution flow.

Importantly, as the next proposition shows, we can write the self-test loss function as a combination of expectations for probability flows. This allows Monte Carlo approximation of the loss function, which is particularly useful for high-dimensional problems when the data consists of sequential ensembles of samples. 
\begin{corollary}
The loss function of \eqref{eq:McKean-Vlasov} with $\nu\geq 0$  can be written as expectations: 
\begin{equation}\label{eq:stLoss-MC}
  \begin{aligned}
 \calE_{u_{[0,T]}}(\Phi,V)  = & \frac{1}{T}\int_0^T \left( \E \big| \, \E [\nabla \Phi (Z_t)|\barX_t] + \nabla V(\barX_t) \big|^2 + 2 \nu \E[\Delta \Phi (Z_t) + \Delta V(\barX_t) ]     \right) dt \\
 & - 2  \big( \E[\Phi(Z_T) + V(\barX_T)] -\E[\Phi(Z_0)+ V(\barX_0) ] \big),
  \end{aligned}
\end{equation}
where $Z_t = \barX_t-\barX_t'$ with $\barX'_t$ is an independent copy of $\barX_t$. 
\end{corollary}
\begin{proof}
	Recall that $u(\cdot,t)$	is the probability density function of $\barX_t$. Hence, we can write the integrals as expectations, for example, $\int_{\R^d} u|\nabla V|^2dx = \E[| \nabla V(\barX_t)|^2]$. In particular, note that $\Phi*u(\barX_t) = \E [ \Phi(\barX_t-\barX_t') |\barX_t]$, where $\barX_t'$ is an independent copy of $\barX_t$. Then, we have $\int_0^T \int_{\R^d}  u | \nabla \Phi * u + \nabla V |^2\, dxdt = \int_0^T \E \big| \, \E [\nabla \Phi (Z_t)|\barX_t] + \nabla V(\barX_t) \big|^2 dt$.  
	
	Meanwhile, note that $\E[ \Phi*u(\barX_t)] = \E\big[ \E [ \Phi(\barX_t-\barX_t') |\barX_t]\big] = \E[\Phi(Z_t)]$. Then, with integration by parts, we can write 
	\[
 \int_{\R^d}  (\Phi*u +  V)   (\partial_t u -\nu \Delta u) \big ] dx  = \partial_t \E[\Phi(Z_t) + V(\barX_t)] - \nu \E[\Delta \Phi(Z_t) + \Delta V(\barX_t)].  
	\]
	Integrate in time over $[0,T]$, we obtain \eqref{eq:stLoss-MC}. 
	\end{proof}

\section{Identifiability and well-posedness}\label{sec:ID}
The quadratic structure of the self-test loss functions provides a framework for analyzing the identifiability of the (function-valued) parameters and the well-posedness of the inverse problems. Notably, since no prior information is assumed for the unknown parameters, we define \emph{adaptive spaces} that depend on both the operator and the data. These spaces capture the limited information available for parameter estimation and provide the appropriate setting for studying the identifiability and well-posedness of the inverse problems.

We demonstrate the approach by estimating $h$, $\Phi$, and $V$ in the operator defined in \eqref{eq:W2-GF}:
\begin{equation*}
R_\phi[u]:= R_{(h,\Phi,V)}[u] =- \nabla\cdot(u \nabla [\nu h'(u) + \Phi*u + V] ) =f. 
\end{equation*}
We start by estimating each parameter individually, assuming the other two are known, in Sections \ref{sec:est_h} and \ref{sec:est_Phi}. Finally, we address the joint estimation of all three parameters. Notably, we establish that the inverse problems for estimating $h$ and $V$ are well-posed, while the estimation of $\Phi$ is ill-posed due to the nonlocal nature of the interaction.

Throughout this section, we construct the parameter spaces using continuum data of input-output pairs $(u_l,f_l)_{l=1}^L$. In practice, discrete data approximates continuum data under appropriate smoothness conditions, as specified in the following assumption. 

\begin{assumption}\label{assum:u}
The data $\{(u_l,f_l)\}_{l=1}^L$ satisfies $f_l = R_{\phi_*}[u_l]$, where $\phi_* = (h_*,\Phi_*,V_*)$, and $\{u_l\}_{l=1}^L \subset \spaceX := C_c^1(\R^d)$, i.e., each $u_l$ has continuous derivatives and compact support.
\end{assumption}

Generalization to non-smooth data $u_l$ is possible in specific settings. For instance, in the absence of the diffusion term (e.g., $\nu=0$), it suffices for each $u_l$ to be a continuous probability density function supported on a compact subset of $\R^d$ for the results in Sections \ref{sec:est_V}--\ref{sec:est_Phi}.

\subsection{Estimating the diffusion rate: well-posed}\label{sec:est_h}
Consider first estimating the diffusion rate function $h:\R^+\to \R $ when $\Phi$ and $V$ are given. We rewrite the equation $R_\phi[u]=f$ to isolate the unknown: 
\begin{equation}\label{eq:R-h''}
	R_{h}[u] := - \nabla\cdot[u \nabla (\nu h'(u)) ]= - \nabla\cdot[\nu h''(u) u \nabla u  ] = f+ \nabla\cdot(u \nabla[ \Phi*u + V] ): = \widetilde f. 
\end{equation} 
Evidently, only $h''$ is identifiable, since $R_h$ depends on $h$ solely through $h''$. Accordingly, we formulate the self-test loss directly in terms of $h''$.

Using the same arguments in Sect.~\ref{sec:stloss-GF}, one can verify that $v_{h}[u]= h'(u)$ is a self-testing function, and the self-test loss function is 
\begin{equation}\label{eq:stloss-h''}
\calE_1(h'')  = \sum_{l=1}^L \innerp{R_h[u_l]-2\widetilde{f}_l,\, v_h[u_l]} + C_0 
\end{equation}
with $C_0$ being an arbitrary constant. Here we used the notation $\calE_1$ to indicate that this loss function is for estimating $h$. 

 Given data $\{u_l\}$ with a compact support, we take the parameter space for $h''$ to be $L^2_{\rho_1}$, where the measure $\rho_1$ is defined through its density function 
\begin{equation}\label{eq:rho-h}
	\dot\rho_1 (r) =  \sum_{l=1}^L \int_{\R^d} \delta(u_l(x)-r)  | \nabla u_l(x) |^2 u_l(x)dx
\end{equation}
with $\delta(\cdot)$ being the Dirac delta function. For any $h''$ in this space, the quadratic term in the loss function is well-defined since    
\begin{align*}
	\sum_{l=1}^L \innerp{R_h[u_l],\, v_h[u_l]} & = \sum_{l=1}^L \int_{\R^d}  u_l(x)  |\nabla u_l(x)|^2 h''(u_l(x) )^2  \, dx = \int_{\R^+} h''(r)^2 \dot\rho_1 (r)  dr,  
\end{align*}
where we used the fact that 
$$
\int_{\R^d}  u(x)  |\nabla u(x)|^2 h''(u(x) )^2 dx = \int_{\R^+} h''(r)^2   \int_{\R^d}  u(x)  |\nabla u(x)|^2 \delta(u(x) -r)  \, dx dr. 
$$

The next proposition presents the well-posedness of estimating $h''$ in  $L^2_{\rho_1}$. 
\begin{proposition} \label{prop:est_h}
Given data $\{(u_l,f_l)\}_{l=1}^L$ satisfying Assumption {\rm \ref{assum:u}}, the self-test loss function in \eqref{eq:stloss-h''} for estimating $h''$ in Eq.\eqref{eq:R-h''} has a unique minimizer in $L^2_{\rho_1}$ with $\rho_1$ defined in \eqref{eq:rho-h}. In particular, the inverse problem of estimating $h''$ is well-posed. 
\end{proposition}

We postpone its proof, along with the proofs of the remaining propositions in this section, to Appendix~\ref{appd:ID}. Since this inverse problem is well-posed, regularization in practice (e.g., Section~\ref{sec:num-h}) serves primarily to smooth the estimator or to filter errors from noise and discretization.

\subsection{Estimating the kinetic potential: well-posed}\label{sec:est_V}
Similarly, we next estimate the potential $V:\R^d\to\R$ assuming that $h$ and $\Phi$ are given. Rewriting $R_\phi[u]=f$ to isolate $V$ gives
\begin{equation}\label{eq:R-V}
  R_{V}[u] := -\nabla\cdot (u\nabla V)
  = f + \nabla\cdot\big(u \nabla[\Phi*u + \nu h'(u)]\big) =: \widetilde f.
\end{equation}
Since $R_V$ depends only on $\nabla V$, we can identify $V$ only up to an additive constant. Accordingly, we formulate the self-test loss directly in terms of $\nabla V$. With $v_{V}[u]= V$ as a self-testing function, the self-test loss function is  
\begin{equation}\label{eq:stloss-V}
\calE_2(\nabla V)  = \sum_{l=1}^L \innerp{R_V[u_l]-2\widetilde{f}_l,\, V} = \|\nabla V\|_{L^2_{\rho_2}}^2 - 2 \innerp{\sum_{l=1}^L \widetilde{f}_l,\, V},
\end{equation}
where we got   
$\sum_{l=1}^L \innerp{R_V[u_l],\, V}= \sum_{l=1}^L \int_{\R^d}  u_l(x)  |\nabla V(x) |^2 \, dx = \|\nabla V\|_{L^2_{\rho_2}}^2$ by integration by parts, and the data-dependent measure $\rho_2$ is defined by its density function 
\begin{equation}\label{eq:rho-V}
	\dot\rho_2 (x) =  \sum_{l=1}^L u_l(x) . 
\end{equation}

The next proposition shows that the inverse problem of estimating $\nabla V\in L^2_{\rho_2}(\R^d,\R^d)$  is well-posed (see its proof in Appendix \ref{appd:ID}). For estimating $V$, the inverse problem is well-posed in 
$\calH_0:=\{g\in H^1_{\rho_2}(\R^d; \R); \int_{\R^d} g \rho_2 \ud x =0\}$ when the measure $\rho_2$ satisfies the Poincare inequality. Here, $H^1_{\rho_2}(\R^d; \R):=\{g\in L^2_{\rho_2}:|\nabla g| \in L^2_{\rho_2}\}$. 
\begin{proposition}\label{prop:est-V}
Consider the problem of estimating $\nabla V$ or $V$ in Eq.\eqref{eq:R-V} from data $\{(u_l,\widetilde{f_l})\}_{l=1}^L$ satisfying Assumption {\rm \ref{assum:u}}. Let $\rho_2$ be the measure defined in \eqref{eq:rho-V}. 
\begin{itemize}
	\item For estimating $\nabla V$, the self-test loss function in \eqref{eq:stloss-V} is uniformly convex and has a unique minimizer in $L^2_{\rho_2}(\R^d,\R^d)$.  Consequently, the inverse problem is well-posed. 
	\item For estimating $V$, assume that $\rho_2$ satisfies the Poincare inequality, i.e., 
        \begin{equation}\label{eq:Poincare}
        \int_{\R^d} |g|^2 \rho_2 \ud x \leq c \int_{\R^d} |\nabla g|^2 \rho_2 \ud x, \quad \forall g\in H^1_{\rho_2} \text{ with } \int_{\R^d} g \rho_2 \ud x =0.
        \end{equation}
         Then the self-test loss function in \eqref{eq:stloss-V}, when viewed as a functional of $V$    in the   space $\calH_0:=\{g\in H^1_{\rho_2}(\R^d; \R); \int_{\R^d} g \rho_2 \ud x =0\}$ is uniformly convex and has a unique minimizer $\widehat V$ satisfying 
         \begin{equation}\label{eq:Vhat}
         -\nabla \cdot (\rho_2 (\nabla \widehat{V}))=   \sum_{l=1}^L \widetilde{f}_l.
         \end{equation}
\end{itemize}
\end{proposition}

 {\color{black} We remark that the assumption of $\rho_2$ satisfying the Poincar\'e inequality in \eqref{eq:Poincare} is mild, and it is equivalent to the spectral gap condition on $\rho_2$ when $\rho_2$ is a probability measure; see, e.g., \cite{BakryGentilLedoux2014}. For instance, $\rho_2(x) = e^{-W(x)}$ with $W\in C^2(\R^d)$ and $\nabla^2 W(x)\geq \frac{1}{c} I_d$ for all $x$, or $\rho_2$ is supported on a bounded connected domain, bounded above, and bounded below away from $0$. 
When $\rho_2$ satisfies this assumption, the potential $V$ can be uniquely recovered up to a constant in $H^1_{\rho_2}$ since it is identifiable in $\calH_0$. However, the minimizer is nonunique without this assumption or beyond $H^1_{\rho_2}$, as shown in the next one-dimensional example. Assume that $d=1$ and $\rho_2(x) = e^{-x}$. Then, if $V$ is a solution to \eqref{eq:Vhat}, so is $V + e^x$ since $\nabla \cdot (\rho_2 (\nabla e^{x})) =0$.}

\subsection{Estimating the interaction potential: ill-posed}\label{sec:est_Phi}
The inverse problem of estimating $\Phi:\R^d\to\R$ differs fundamentally from the previous two, as it is ill-posed due to its deconvolution structure. Here, we estimate $\Phi$ in
\begin{equation}\label{eq:R-Phi}
  R_{\Phi}[u] := -\nabla\cdot \big(u\nabla (\Phi*u)\big) 
  = f + \nabla\cdot\big(u \nabla[ V + \nu h'(u)]\big) =: \widetilde f,
\end{equation}
where $V$ and $h'$ are given. As $R_\Phi$ depends on $\Phi$ only through $\nabla \Phi$, we can identify $\Phi$ only up to an additive constant.

 Denote $F[u]=  \Phi_**u + V_*-V + \nu h_*'(u)- \nu h'(u)$ and note that $\widetilde f =  - \nabla\cdot(u \nabla[F[u]] )$. 
 With a self-testing function $v_{\Phi}[u]= \Phi*u$, the self-test loss function of $\nabla \Phi$ is 
\begin{equation}\label{eq:stloss-Phi} 
\begin{aligned}
\calE_3(\nabla \Phi) & = \frac{1}{2}\sum_{l=1}^L \innerp{R_\Phi[u_l]-2\widetilde{f}_l,\, \Phi*u} \\
                                 & = \frac{1}{2} \sum_{l=1}^L \int_{\R^d}  u_l \big( |\nabla \Phi*u |^2  - 2  \nabla F[u_l] \cdot \nabla \Phi*u_l \big) dx. 
	\end{aligned}
\end{equation}

The independent variable of $\Phi$ is the pairwise difference of particle positions, while the data $u$ is the distribution of each particle. To quantify the exploration of the independent variable of $\Phi$ by data, we define a measure $\rho_3$ with a density function 
\begin{equation}\label{eq:rho-Phi}
	\dot\rho_3 (y)  \propto \sum_{l=1}^L \int {u_l(x)}  u_l(x-y) dx. 
\end{equation} 
It extends the exploration measure defined in \cite{LangLu21id,LangLu22} for radial interacting potentials. 

Let
$L_\Gbar:L^2_{\rho_3}\to L^2_{\rho_3}$ be an integral operator defined by 
\begin{equation}\label{eq:L_Gbar}
\begin{aligned} 
L_\Gbar \nabla \Phi   (y)  = &  \int \Gbar(y,y')  \nabla \Phi(y')\rho_3(dy') \, \text{ with }  
\Gbar(y,y')  =  \frac{G(y,y')}{\dot \rho_3(y)\dot \rho_3(y')} \mathbf{1}_{\dot\rho_3(y)\dot \rho_3(y') >0}, \\
G(y,y') =&  \sum_{l=1}^L \int  u_l(x)   u_l(x-y) u_l(x-y')  dx. 
\end{aligned}
\end{equation}
Here $\Gbar(y,y')$ is square integrable by Assumption \ref{assum:u}; see \cite{LangLu21id}.

The next proposition shows that we can only identify $\nabla \Phi \in \mathrm{Null}(L_\Gbar)^\perp \subset L^2_{\rho_3}$, and the inverse problem of estimating $\nabla \Phi$ is ill-posed. 
\begin{proposition}\label{prop:est-Phi}
  Consider the problem of estimating $\nabla \Phi$ in Eq.\eqref{eq:R-Phi} from data $\{(u_l,f_l)\}_{l=1}^L$ satisfying Assumption {\rm \ref{assum:u}}. Let $\rho_3$ be the measure defined in \eqref{eq:rho-Phi}. The Hessian of the quadratic self-test loss function in \eqref{eq:stloss-Phi} is the compact operator $L_\Gbar$ on $L^2_{\rho_3}(\R^d,\R^d)$ defined in \eqref{eq:L_Gbar}. Consequently, the inverse problem of finding its minimizer in \eqref{eq:Phi-hat-argmin} is ill-posed. 
  \end{proposition}

  A regularization is necessary to obtain a stable solution for this ill-posed inverse problem of estimating $\nabla \Phi$. In particular, when $\mathrm{Null}(L_\Gbar)\neq \{0\}$, it is crucial to regularize only on $\mathrm{Null}(L_\Gbar)^\perp$ in order to prevent the estimator from being contaminated by components in $\mathrm{Null}(L_\Gbar)$. Data-adaptive RKHS regularization or priors, as proposed in \cite{LLA22,chada2024data}, employ the RKHS with reproducing kernel $\Gbar$ and yield convergent estimators.

  The ill-posedness in estimating $\nabla \Phi$ stems from the deconvolution structure of the problem. Consequently, even if additional properties are imposed on $\Phi$, such as radial or symmetry, the inverse problem remains ill-posed. However, when the data $u_l$ are contaminated by additive spatial noise, the operator $L_\Gbar$ in \eqref{eq:L_Gbar} can become strictly positive definite, which in turn yields a well-posed inverse problem; see Section~\ref{sec:num-Phi} for a numerical illustration.

\subsection{Joint estimation}\label{sec:est_all3}
Using the parameter spaces and operators in the previous sections, the joint estimation for $(h'',\nabla V,\nabla \Phi)$ takes place in the product space $L^2_{\rho_1}(\R^+)\otimes L^2_{\rho_2}(\R^d)\otimes L^2_{\rho_3}(\R^d)$. Meanwhile, the self-test loss function can be written as 
\begin{equation}\label{eq:stloss-jiont}
\begin{aligned}
    \calE(h'',\nabla V,\nabla \Phi)   
  : = &   \sum_{l=1}^L \int_{\R^d} \big[ u_l | \nabla [  \nu  h'(u_l) + \Phi * u_l + V)] |^2 \\
  &  - 2 f_l [ \nu h'(u_l) +\Phi*u_l + V ]  \big ] dx\ dt.%
\end{aligned} 
\end{equation}

The next proposition shows the ill-posedness of estimating  $(h'',\nabla V,\nabla \Phi)$.  
\begin{proposition}[Joint estimation]
  \label{prop:est-joint}
  Consider the problem of jointly estimating $h'',\nabla V,\nabla \Phi$ in Eq.\eqref{eq:W2-GF} from data $\{(u_l,f_l)\}_{l=1}^L$ satisfying Assumption {\rm \ref{assum:u}}. Let $\rho_1,\rho_2, \rho_3$ be the measures defined in \eqref{eq:rho-h},\eqref{eq:rho-V} and \eqref{eq:rho-Phi}, respectively. Then, the self-test loss function in \eqref{eq:stloss-jiont} is not uniformly convex, and its Hessian (second variation) has a zero eigenvalue with eigenfunction $\phi=(0,\mathbf{c},-\mathbf{c})$ for any nonzero $\mathbf{c}\in \R^d$. In particular,  the joint estimation problem of finding the minimizer of the loss function is ill-posed. 
\end{proposition}

We remark that the singular value of the loss function's Hessian roots in the fact that different pairs $(\Phi, V)$ and $(\Phi + \mathbf{c} \cdot x, V - \mathbf{c} \cdot x)$ produce the same value of the loss function, which has been noticed in \cite{yao2022mean}. To eliminate this degeneracy, we enforce symmetry on $\Phi$, so that vectors of the form $(0,\mathbf{c},-\mathbf{c})^T$ no longer lie in the admissible function space for $\Phi$. In practice, this constraint can be implemented either by restricting to radial potentials (see Section~\ref{sec:num-Phi}) or by parametrizing $\Phi$ via $\Phi(x) = \widetilde{\Phi}(x) + \widetilde{\Phi}(-x)$, where $\widetilde{\Phi}$ is a learnable function (e.g., a neural network as in Section~\ref{sec:num-PhiV}).

\section{Applications to parametric and nonparametric estimations}\label{sec:numerics}
We demonstrate applications of the self-test loss function in estimating the function parameters in the weak form operator $R_{(h,\Phi,V)}[u] =- \nabla\cdot(u \nabla [\nu h'(u) + \Phi*u + V] )$ in \eqref{eq:W2-GF} and its gradient flow. We consider parameter estimation for $h$ in Section \ref{sec:num-h}, nonparametric estimation for radial $\Phi$ in Section \ref{sec:num-Phi}, and neural network regression for joint estimation of $\Phi$ and $V$ in Section \ref{sec:num-PhiV}. 


\subsection{Parametric estimation of the diffusion rate function}\label{sec:num-h}
Consider first a parametric estimation of $h$ in the equation 
\begin{equation}\label{eq:Ru_f_5_1}
    R_h[u]:=-\nabla\cdot(u [\nabla  h'(u)] =  - \nabla\cdot [ u h''(u) \nabla u]  =  f,
\end{equation}
from data $\{(u_l(x_i),f_l(x_i))\}_{i=1,l=1}^{N,L}$, where $x_i\in [0,1]$ is a uniform mesh and $u_l\in H^1_0((0,1))$. Here, the diffusion rate function $h$ is a power function in \eqref{eq:h-powerFn} with a parametric form
\begin{equation}\label{eq:h-para}
h_{\mathbf{c}}(s)  = 
 c_2  s^2 + c_3\frac{1}{2} s^3 + c_4 \frac{1}{3}s^4 = \sum_{k=1}^{n_c} c_k e_k(s), 
\end{equation}
where 
 $e_k(s) = \frac{1}{k-1}s^{k}$ for $k>1$, and $n_c=3$. 
Thus, the task is to estimate the parameters $\mathbf{c}=(c_2,c_3,c_4)$. Here we don't include the term $e_{1}(s)= s\log s$ because its second derivative $e_1''(s) =1/s$ is singular at $s=0$. Such a singularity leads to a singular function $e_1''(u_l(x))$ when $u_l(x)$ approaches zero at the boundaries, requiring additional numerical treatments when computing the loss function of $h''$ and the normal matrix for regression. 

\noindent\textbf{Synthetic Data generation.}
We generate data by adding noise to the values of analytically computed functions on the spatial mesh. Let the mesh points be  $x_i = \{ \frac{j}{N}, 1\leq j\leq N\}$. We obtain noisy data $\{u_l(x_i)\}$ by adding independent Gaussian noises $\mathcal{N}(0,\sigma^2/N)$ to the values of $u_l(x) =  \sin(\pi l x )$ on the mesh for $l\in \{1,2,3\}$. Note that these functions are in $H_0^1((0,1))$.  

The data $\{ f_l(x_i)\}$ are noisy observations of $R_{h_{\mathbf{c}^*}}[u_l](x)$ at the meshes: 
 $$f_l(x_i) =- R_{h_{\mathbf{c}^*}}[u_l](x_i) + \epsilon_{l,i} =- \sum_{k=2}^{n_c} c_k  \nabla\cdot [ u_l e_k''(u_l) \nabla u_l](x_i) + \epsilon_{l,i} $$ 
 with parameter $\mathbf{c}^*= (c_2,c_3,c_4) = (1,1.2,0.5)$ and $\{\epsilon_{l,i}\}$ being i.i.d. $\mathcal{N}(0,\sigma^2/N)$. Here we compute each $\nabla \cdot [u_l e_k''(u_l)\nabla u_l]$ analytically since $e_k$ and $u_l$ are polynomials and trigonometric functions.   

\noindent\textbf{Regression from the self-test loss function.} As studied in Section \ref{sec:est_h}, the self-test loss function in \eqref{eq:stloss-h''} with Riemann sum approximation is 
\begin{align*}
	\calE(\mathbf{c}) & =  
  \frac{1}{NL}\sum_{l=1}^L \sum_{i=1}^N \big[  h_\mathbf{c}''(u_l(x_i))^2 u_l(x_i) |\nabla u_i(x_i)|^2 - 2 f_l(x_i) h_\mathbf{c}'(u_l(x_i)) \big] \\ 
  & = \mathbf{c}^\top \mathbf{A} \mathbf{c} - 2\mathbf{c}^\top \mathbf{b}, 
\end{align*}
where the normal matrix $\mathbf{A}= (A_{k,m})$ and normal vector $\mathbf{b}$ defined by 
 \begin{align*}
  A_{k,m} &= \frac{1}{NL} \sum_{l=1}^L \sum_{i=1}^N  u_l(x_i) |\nabla u_l(x_i)|^2 e_k''(u_l(x_i)) e_m''(u_l(x_i)), \quad 1\leq k,m\leq n_c\\
  b_k &= \frac{1}{NL} \sum_{l=1}^L \sum_{i=1}^N f_l(x_i) e_k'(u_l(x_i)), \quad 1\leq k\leq n_c. 
  \end{align*}
The estimator is then solved by least squares regression,  
\[
\widehat h (s) = \sum_{k=1}^{n_c} \widehat c_k e_k(s), \quad \widehat{\mathbf{c}} =  \mathbf{A}^{-1} \mathbf{b}. 
\]

 We compare $\widehat h (s)$  (denoted by ``stLoss-estimator'') with an estimator using the strong-form equation (denoted by ``Strong-estimator''). The strong form estimator has coefficient $\widehat{\mathbf{c}}^s =  (\mathbf{A}^s)^{-1} \mathbf{b}^s$, where the normal matrix $\mathbf{A}^s$ and normal vector  $\mathbf{b}^s$ have entries 
 \begin{align*}
 A^s_{k,m} &=\frac{1}{NL} \sum_{l=1}^L \sum_{i=1}^N \nabla\cdot [ u_l e_k''(u_l) \nabla u_l](x_i)  \nabla\cdot [ u_l e_m''(u_l) \nabla u_l](x_i), \\
 b^s_k&= \frac{1}{NL} \sum_{l=1}^L \sum_{i=1}^N f_l(x_i)  \nabla\cdot [ u_l e_k''(u_l) \nabla u_l](x_i). 
 \end{align*} 
 Thus, the strong form estimator uses the second-order derivatives of $u$, while the weak form estimator uses only the first-order derivatives. 

 In the computation of both estimators, the derivatives are approximated by finite difference using the Savitzky-Golay filter with polynomial degree 3 and window size 11 (see, e.g.,\cite{schafer2011savitzky}). The difference between the two is that the Strong-estimator requires an additional finite difference approximation for the divergence term, whereas the stLoss-estimator uses the Riemann sum to approximate the integration. 

 \begin{figure}
\includegraphics[width=\textwidth]{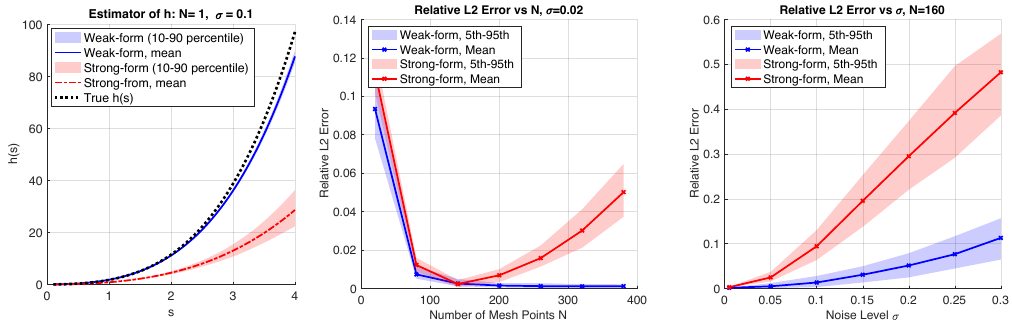}	
\caption{Estimators from self-test loss function (``stLoss'') vs estimators from strong-form equation (``Strong''). \textbf{Left:}  estimators in a typical set of 100 simulations with $N=400$ and $\sigma=0.1$. \textbf{Middle-Right:} Relative $L^2_{\rho_1}$ errors vs $N$ and $\sigma$.} \label{fig:est-h}\vspace{-3mm}
\end{figure}
      
Figure \ref{fig:est-h} shows that the stLoss-estimator significantly outperforms the Strong-estimator in typical simulations ({Left}) and in the convergence of relative $L^2_{\rho_1}$ error ($\|h_{*}''- \widehat h ''\|_{L^2_{\rho_1}}^2$) as mesh size $N$ increases or as the noise level $\sigma$ vanishes ({Middel-Right}). Note that as $N$ increases, the strong-estimator does not converge due to the noise being amplified by the additional finite difference approximation (recall that the noise decays at the rate $O(\frac{1}{\sqrt{N}})$ while the $\Delta x$ in finite difference has an order of $O(\frac{1}{N})$). 

Here, for each parameter set of $(N,\sigma)$, the percentiles are computed from 100 independent simulations with randomly generated noise; the empirical measure $\rho_1$ in \eqref{eq:rho-h} is computed from data by Riemann sum approximation of the integral and finite difference approximation of the derivatives. In the typical simulation (Left), the condition numbers of the normal matrix $\mathbf{A}$ are in the range $[30,40]$, indicating the well-posedness of the inverse problem.

     \bigskip 
In summary, the example shows that the estimator using our self-test loss function based on the weak-form equation can tolerate a rougher spatial mesh and larger-scale noise in the data than a strong-form-based estimator.

\bigskip 
{\color{black}
\noindent\textbf{Application to weak SINDy.} We apply our self-test loss function within the weak SINDy framework of \cite{messenger2021weak} to estimate the sparse parametric diffusion rate $h$ in \eqref{eq:h-para} from data. Specifically, we compare our self-testing functions with random Gaussian test functions $\{\psi_m\}_{m=1}^M$ designed following the strategy in \cite{messenger2021weak}, namely, tailoring them to the noise level and spectral properties of the data.  This design keeps the test set computationally feasible while avoiding the curse of dimensionality that affects more structured families. The Gaussian test functions have centers sampled uniformly in $[0,1]^d$ (with $d=2$) and bandwidths $\eta \in \{0.025, 0.1, 0.4\}$.

We assuming the true coefficient for $h_c(s)$ is given by 
$$ \bold{c} = (c_1, c_2, c_3, c_4, c_5) = (1, 0, 2, 0, 0).$$
The data is generated on the discrete mesh, and for $d\geq 1$, we consider the data to be the tensor product $u_l(x) = \sin(\pi l x_1) \cdots \sin (\pi l x_d)$ evaluated over a discrete mesh with $N = 100$ grids in each dimension, where $x = (x_1, \dots, x_d) \in \R^d$ and $1\leq l \leq L$ with $L=2$. The data $u_l$ and $f_l$ at each mesh point are polluted with additive Gaussian noise of variance $\sigma = 0.05$. Since the true parameter is known to be sparse, we use the modified sequential-thresholding least-squares (MSTLS) as introduced in \cite{messenger2021weak} to promote sparsity in the estimation. 

We report the estimation error as a function of $M$ over 10 independent simulations in Figure \ref{fig:SINDy_1}. In the self-test setting, we always use $L=2$ test functions, so the corresponding error is independent of $M$. As $M$ increases, the error obtained with random Gaussian test functions decreases. The best performance occurs when the Gaussian bandwidth $\eta=0.1$ is close to the noise level $\sigma=0.05$, in which case the errors approach those achieved by the self-testing functions. Even in this near-optimal regime, the self-testing functions still outperform the random Gaussian test functions. This example highlights the advantage of the proposed self-test framework.

Separately, additional experiments (not shown) suggest that self-testing functions yield an estimation error on the order of the noise level. In contrast, with sufficiently large $M$, random Gaussian test functions can produce errors below the noise level. This indicates that, with an appropriate choice of bandwidth $\eta$, the random Gaussian test functions may effectively filter out the noise. 
}

\begin{figure}[t]
\centering
\includegraphics[width=\textwidth]{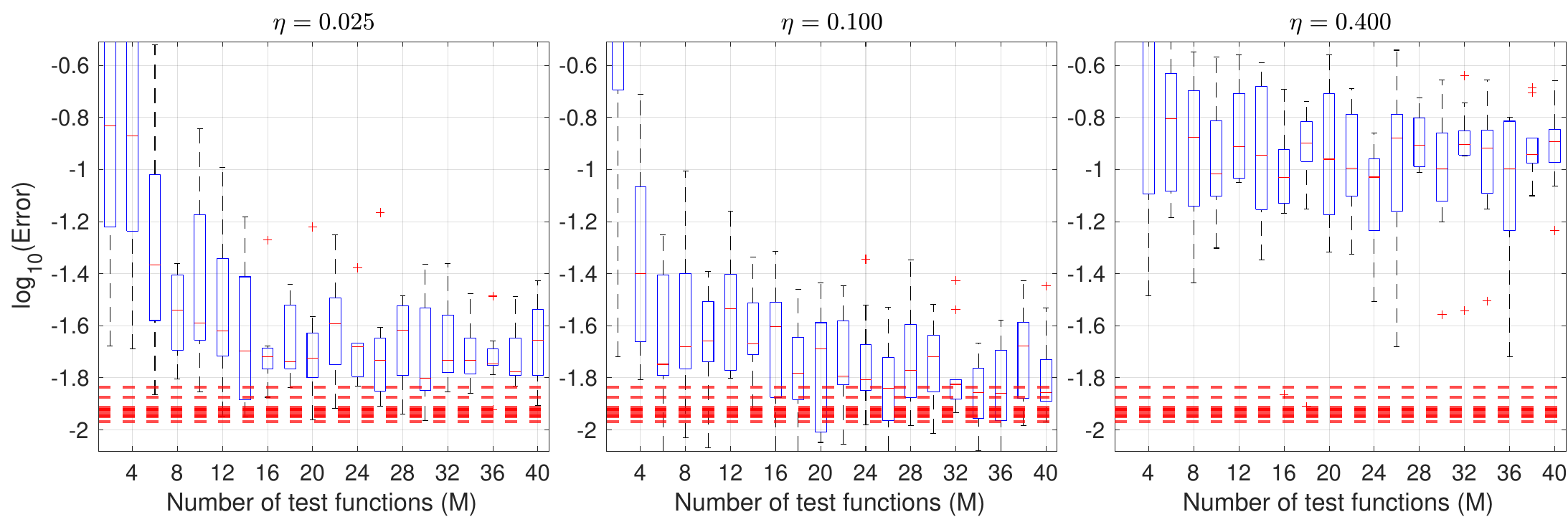}	
\caption{Comparison between self-testing functions and random Gaussian test functions in the weak SINDy framework. The boxplots show the distribution of estimation errors over 10 independent simulations using $M$ random Gaussian test functions, with bandwidths $\eta \in \{0.025, 0.1, 0.4\}$. The red dashed lines indicate the error obtained using $L=2$ self-testing functions, which do not depend on $M$. As $M$ increases, the error for Gaussian test functions decreases, and the best performance occurs when $\eta = 0.1$, close to the noise level $\sigma = 0.05$, at which point the errors approach those of the self-test formulation. Even in this near-optimal setting, self-testing functions yield consistently lower errors than the random Gaussian tests.}
\label{fig:SINDy_1}
\end{figure}

\subsection{Non-parametric estimation of interaction kernel}\label{sec:num-Phi}
Next, we consider estimating an interaction kernel, the derivative of a radial interaction potential, in the aggregation operator. We will compare strong-form and weak-form estimators with respect to their tolerance to observation noise.  
 
Specifically, consider the estimation of the function $\phi:[0,2]\to\R $, which is the derivative of the radial interaction potential $\Phi$ with $\nabla \Phi(x) =\phi(|x|)\frac{x}{|x|} $, in the aggregation operator 
\[R_\phi[u] = -\nabla  (u\nabla \Phi *u ) = f,  \]
from  data consisting of noisy input-output function pairs at discrete meshes in 1D:
\begin{equation}\label{eq:uf_disc_noisy}
	\{ (u_{k}^o(x_j):=u_{k}(x_j)+\epsilon_{kj}^u, f_{k}^o(x_j) = f_k(x_j)+\epsilon_{kj}^f )\}_{j=1,k=1}^{n_x,n_k},  
\end{equation}
where $ \epsilon_{kj}^u,\epsilon_{kj}^f  \sim \mathcal{N}(0,\sigma^2)$. 
Here, the spatial meshes $\{x_j\}$ are uniform on $\Omega = (0,10)$ satisyfing $x_{j}-x_{j-1} =\Delta x = 0.01$ for all $j$, and the noises $\{\epsilon_{kj}^u,\epsilon_{kj}^f \}_{k,j=1}^{n_k,n_x}$ are independent identically distributed Gaussian $\mathcal{N}(0,\sigma^2)$ random variables with standard deviation $\sigma$. The functions $\{u_k\}$ are $u_k(x) = \sin(\pi(x-(2k+1)))\mathbf{1}_{\{ |x-(2k+1)|<1.5\}}$ for $1\leq k\leq  n_k=3$. They are in $C_c^1(\Omega)$, so we can use integration by parts in the weak form and compute $f_k(x_j)$ using the strong form operator, i.e., we compute the analytical form of the integrand in the following integral,  
\begin{align*}
	f_k(x) &= - \int_0^2 \phi_{*}(|y|) \mathrm{sign}(y)\partial_x [u_k(x-y)u_k(x)]dy , 
\end{align*}
 where the integral is computed using the adaptive Gauss-Kronrod quadrature \cite{shampine2008vectorized}. In our tests, we set $\phi_{*}(r)= r^2\mathbf{1}_{[0,1]}(r)$. 
 Figure \ref{fig:est-Phi}(a) shows the data pairs.

The above equation is the mean-field equation \eqref{eq:FPE-V-Phi} with $V=0$ if $f= \pt_t u -\nu \Delta u$. In this case, the nonparametric estimation of $\phi$ has been studied in \cite{LangLu21id,LangLu22}. Here, we focus on the aggregation operator without the diffusion term.   

 In the following, we derive the least squares regression of $\phi$ using the self-test loss function. We first write the self-test loss function in the continuum, then approximate it by the discrete data and write the least squares estimator of $\phi$.

\noindent\textbf{The self-test loss function in continuum.} Using the self-testing function $v_\phi[u_k] =\Phi*u_k $ for each input-output pair $(u_k,f_k)$, and applying integration by parts, we obtain the self-test loss function  
\begin{align*}
  \calE_\mathcal{D}(\phi) = \sum_{k=1}^{n_k} \int_{\R}  u_k | \nabla \Phi * u_k|^2 dx  - 2  \int_{\R} f_k(x) \Phi*u_k(x) dx. 
\end{align*} 
Denote $F_{f,u}(r) := -\sum_{k=1}^{n_k} \int_0^{10} F_k(x)u_k(x)[u_k(x-r)- u_k(x+r)]\,dx$ with $F_k(x) :=\int_0^x f_k(y)dy$. 
We can write the loss function as (see Appendix \ref{append:num-phi-details} for a derivation) 
\begin{equation}\label{eq:loss_phi_radial}
\begin{aligned}
  \calE_\mathcal{D}(\phi) 
  = & \int\int \phi(r)\phi(s) \Gbar(r,s)\dot \rho(r)\dot \rho(s)drds  
  & - 2  \int_0^2 \phi(r) F_{f,u}(r)dr, 
\end{aligned} 
\end{equation}
where the density of the exploration measure $\dot\rho$ is defined as    
\begin{equation}\label{eq:rho_radial}
\dot\rho(r) := \frac{1}{Z}\sum_{k=1}^{n_k}\int_\R \sqrt{u_k(x)}|\delta u_k(x,r)| \,dx\,  \text{ with } \delta u_k(x,r) := u_k(x-r)- u_k(x+r)  
\end{equation}
with $Z$ being a normalizing constant. 
Here, the integral kernel $\Gbar$ is defined by 
\begin{equation}\label{eq:GGbar_radial}
\begin{aligned}
\Gbar(r,s) : = \frac{G(r,s)}{\dot\rho(r)\dot\rho(s)} \mathbf{1}_{\{\dot\rho(r)\dot\rho(s)>0\}}\, \text{ with }  G(r,s) : = \sum_{k=1}^{n_r}  \int_\R u_k(x)\delta u_k(x,r)  \delta u_k(x,s)dx  \,. 
\end{aligned}
\end{equation}
 These integrals are well-defined since $\{u_k(x)\}$ are uniformly bounded with compact support.

\begin{figure}
\includegraphics[width=0.95\textwidth]{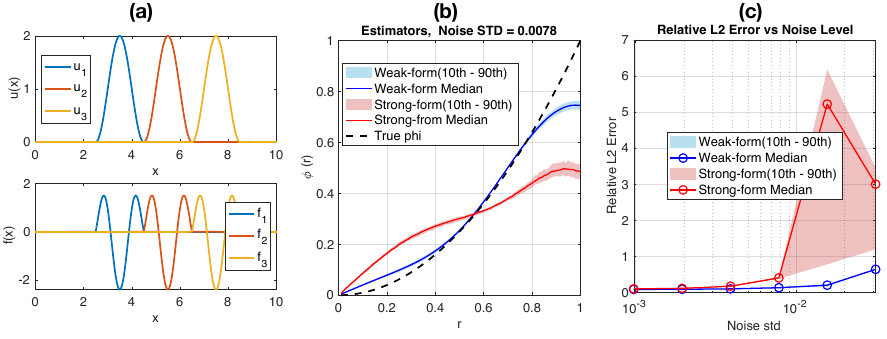}	
\caption{Estimators using weak (``Weak-form'') vs strong-form (``Strong-form'') equation. \textbf{(a)}:  Dataset $\{(u_k,f_k)\}_{k=1}^3$. 
\textbf{(b)}: Estimators (in percentiles) in a typical set of 100 simulations with noise level $\sigma = 0.0078$. \textbf{(c)} Relative $L^2$ errors v.s.~noise level $\sigma$.  Weak-form estimators are more robust to large noise than those based on the strong-form.  } \label{fig:est-Phi} \vspace{-3mm}
\end{figure}

\noindent\textbf{Least squares regression from empirical loss function.} Given the discrete data in \eqref{eq:uf_disc_noisy} on the mesh $\{x_j\}$, we can obtain a uniform mesh $\{r_l = l\Delta x \}_{l=1}^{n_r}$ on $[0,2]$ for the independent variable of $\phi$. Representing $\phi$ by a linear combination of piecewise constant functions $\phi(r) =\sum_{l=1}^{n_r} c_l\mathbf{1}_{[r_l,r_{l+1}]}(r)$,  
our task is to estimate the coefficient vector $\mathbf{c} = (c_1,\ldots,c_{n_r})^\top \in \R^{n_r\times 1}$. Approximating the loss function in \eqref{eq:loss_phi_radial} by Riemann sum using the noisy data and the above piecewise constant $\phi$, we obtain an empirical loss function that is quadratic in $\mathbf{c}$: 
 \begin{align*}
  \widehat{\calE_\mathcal{D} }(\mathbf{c}) = \mathbf{c}^\top \mathbf{A} \mathbf{c}- 2 \mathbf{c}^\top \mathbf{b} + C, 
\end{align*} 
where $\mathbf{A}\in \R^{n_r\times n_r}$ and $\mathbf{b}\in \R^{n_r\times 1}$ are the normal matrix and vectors and $C$ is a constant term independent of $\mathbf{c}$. The entries of $\mathbf{A}$ and $\mathbf{b}$ are 
\begin{align*}
  \mathbf{A}(l,l') & = \mathbf{G}(l,l') \approx \int\int \mathbf{1}_{[r_l,r_{l+1}]} (r)\mathbf{1}_{[r_{l'},r_{l'+1}]} (s) G(r,s)drds \\
  \mathbf{b}(l)    & =  -  \mathbf{g}^\top \widetilde{\mathbf{f} } \Delta r \Delta x 
                                \approx \int \mathbf{1}_{[r_l,r_{l+1}]} (r) F_{f,u}(r) dr,
\end{align*} 
where we denote $\mathbf{G} = \mathbf{g}^\top \mathbf{g} \Delta x (\Delta r)^2\in \R^{n_r\times n_r}$ with $\mathbf{g} = \big(\sqrt{|u_k^o(x_j)|} \delta u_k^o(x_j,r_l) \big) \in \R^{n_kn_x \times n_r} $ and $\widetilde{\mathbf{f} } = \big(\sum_{i=1}^jf_k^o(x_i) u_k^o(x_j) \Delta x \big)_{k,j=1}^{n_k,n_x}\in \R^{n_kn_x\times 1}$. 

The estimator is then solved with Tikhonov regularization: 
\[\widehat \phi(r) =\sum_{l=1}^{n_r} \widehat c_l\mathbf{1}_{[r_l,r_{l+1}]}(r), \quad (\widehat c_1,\ldots, \widehat c_{n_r})^\top :=  \widehat{\mathbf{c}}_{\lambda_*} = \big(\mathbf{A}+ \lambda_* \mathbf{I} \big)^{-1} \mathbf{b}
\]
with the hyperparameter $\lambda_*>0$ selected by the L-curve method \cite{hansen00_Lcurve}.  Due to the additive noise in $u_k^o$, the smallest eigenvalue of the normal matrix $\mathbf{A}$ is bounded below by a constant that scales with $\sigma^2$. Thus, the noise prevents $\mathbf{A}$ from being severely ill-conditioned, and the regularization mainly acts as a filter of the noise. Here we regularize using norm $\|\mathbf{c}\|_{\R^{n_r}}^2 = \mathbf{c}^\top \mathbf{I} \mathbf{c}$, and we leave it in future work to investigate other norms, such as the $L^2_\rho$-norm or the data-adaptive RKHS norm of $\phi$ in \cite{LLA22}.  

Also, we compute the exploration measure as $\brho = (\dot\rho(r_1),\ldots,\dot\rho(r_{n_r})) \in \R^{n_r\times 1}$ with $\dot\rho(r_l) = \frac{1}{Z}\sum_{k=1}^{n_k}\sum_{i=1}^{n_x} \sqrt{|u_k^o(x_i)|} |\delta u_k^o(x_i,r_l)| \Delta x$. The $L^2_\rho$ norm of $\phi$ is then given by $\|\phi\|_{L^2_\rho}^2 = \sum_{l=1}^{n_r} c_l^2\dot\rho(r_l)$. 

\noindent\textbf{Numerical results.} 
We compare the estimators using the weak-form and strong-form equations. The strong-form estimator uses the Savitzky-Golay filter to compute derivatives. We compute the estimators from data with noise levels $\sigma = \{2^{-j}, j=5,\ldots, 10\}$. We make 100 independent simulations for each noise level, each with randomly sampled noise.   

Figure \ref{fig:est-Phi} \textbf{(b)-(c)} reports the estimators and relative errors using the median, the 10th and 90th percentiles. In particular, \textbf{(b)} shows that the weak-form estimator is more accurate than the strong-form estimator when the noise level is $\sigma = 2^{-7}\approx 0.0078$.  \textbf{(c)} shows that when the noise level is small, the strong-form estimator is as accurate as the weak-form, indicating the effectiveness of the Savitzky-Golay filter. Still, when the noise level is high, the strong-form estimator has larger errors than the weak-form estimator, due to the need to approximate derivatives using finite differences. 

In summary, the weak-form estimator outperforms the strong-form estimator in terms of robustness to high levels of noise.

\subsection{Neural network regression for joint estimation}\label{sec:num-PhiV}
This section considers the joint estimation of the interaction potential $\Phi$ and the potential $V$ of the deterministic interacting particle system in Example \ref{exp:IPS} from sequential ensembles of unlabeled data. We use the self-test loss function in \eqref{eq:loss_MC_deterministic} for the weak form PDE of the empirical measures, as derived in Section \ref{sec:ensembleData}.

\noindent\textbf{Numerical settings.}
In our test, we set $M = 10$, $N = 30$, $d = 2$, and $t_l= l\Delta t$ with $\Delta t = 0.01$ and $L=20$. The particle system is solved using the fourth-order Runge-Kutta method. The true interaction and external force potentials are given by
\begin{equation}\label{eq} 
\Phi^*(x) = \cos(2x_1^2) + \cos(x_2), \quad V^*(x) = \exp\left(-\frac{3}{10}\left(\sin(2 x_1)^2 + \arctan(x_2)\right)\right). 
\end{equation}

In the data in \eqref{eq:data_Xtl}, the initial conditions $( X_{t_1}^{i, (m)}, 1\leq i\leq N)\in \R^{Nd}$ are randomly sampled, half of samples from the uniform distribution over $ [-2, 2]^{Nd} $ and the other half from a Gaussian mixture, so that the data spreads out in a region. Here $d=2$ and the Gaussian mixture is the product measure of the distribution $0.6 \times \mathcal{N}(\mu_1, \Sigma_1) + 0.4 \times \mathcal{N}(\mu_2, \Sigma_2)$ on $\R^2$, 
where $\mu_1$ are sampled from a uniform distribution on $[0, 2.5]^2$ and $\mu_2$ are sampled from a uniform distribution on $[-2.5, 0]^2$. The covariance matrices are fixed to be $\Sigma_1 = \begin{pmatrix}
    0.2 &0\\0 &0.4
\end{pmatrix}$ and $\Sigma_2 = \begin{pmatrix}
    1 &0.5\\0.5&1
\end{pmatrix}$.
In this setting, the distribution of the particles is concentrated in the first and third quadrants, as shown in  Figure \ref{Fig:joint_f}.

\noindent\textbf{Regression via neural network approximation.}
We use neural networks to approximate both the interaction and external force potentials. To approximate the interaction and external force potentials, we use two four-layer fully connected neural networks with sigmoid and \texttt{ReLU} activation functions. In particular, we enforce symmetry by setting $\Phi(x) = \tilde{\Phi}(x) + \tilde{\Phi}(-x)$, where $\tilde{\Phi}$ is the neural network approximation. This constraint resolves the identifiability issue in Proposition \ref{prop:est-joint} and in \cite{yao2022mean}, where different pairs $(\Phi, V)$ and $(\Phi + \mathbf{c} \cdot x, V - \mathbf{c} \cdot x)$ produce the same value of the loss function, since $\Phi + \mathbf{c}\cdot x$ is only symmetric when $\mathbf{c} = 0$ if $\Phi$ is symmetric. 

Optimization is performed using the Nesterov-accelerated Adaptive Moment Estimation \texttt{NAdam} method, which combines Adam's adaptive learning rates with Nesterov's lookahead mechanism to improve convergence and optimization efficiency \cite{dozat2016incorporating}, with a learning rate adjustment \texttt{ReduceLROnPlateau}, which reduces the learning rate when a monitored metric stops improving, helping to fine-tune optimization and avoid overfitting. 

The training process is presented in Figure \ref{fig:training_process}. The initial step size is set to be $\eta = 0.05$, and it is reduced to $0.1\eta$ whenever the loss stops reducing. The final minimized loss is -0.001309. Note that our self-test loss \eqref{def:self-testFn} is the quadratic \eqref{eq:st-Loss-square} minus a constant, where the constant is related to the true functions. The true constant in this example is 0.001377, which suggests that the quadratic loss has been minimized to $6.69\times 10^{-5}$.

\begin{figure}[t]
    \centering
    \begin{tabular}{cccc}
        \begin{subfigure}[b]{0.5\textwidth}
            \centering
            \includegraphics[width=\textwidth]{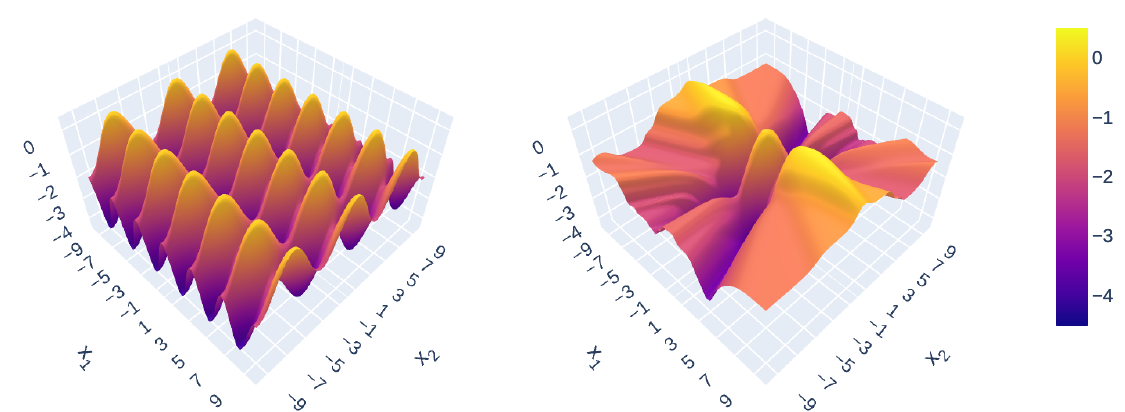}  
            \caption{True and estimated interaction potential $\Phi$}
            \label{Fig:joint_a}
        \end{subfigure} &
        \begin{subfigure}[b]{0.21\textwidth}
            \centering
            \includegraphics[width=\textwidth]{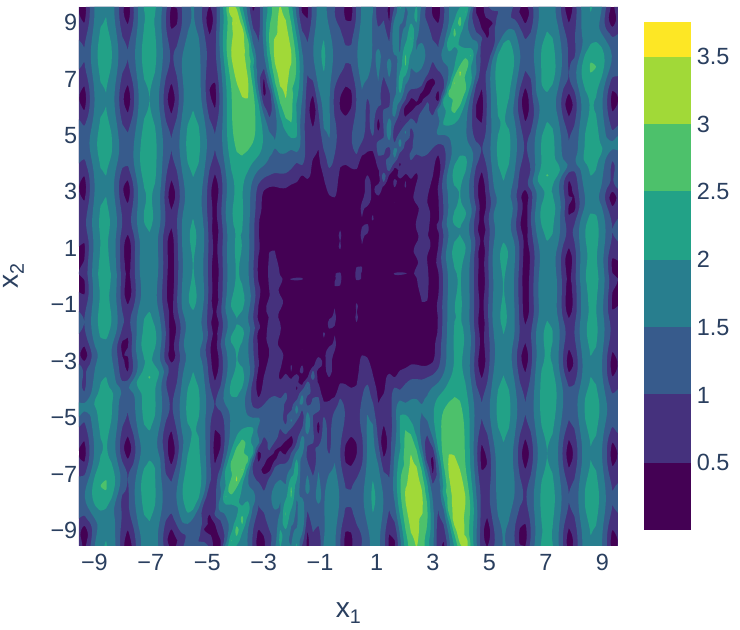}  
            \caption{$|\nabla \Phi - \nabla \widehat{\Phi}|$}
            \label{Fig:joint_b}
        \end{subfigure} &
        \begin{subfigure}[b]{0.21\textwidth}
            \centering
            \includegraphics[width=\textwidth]{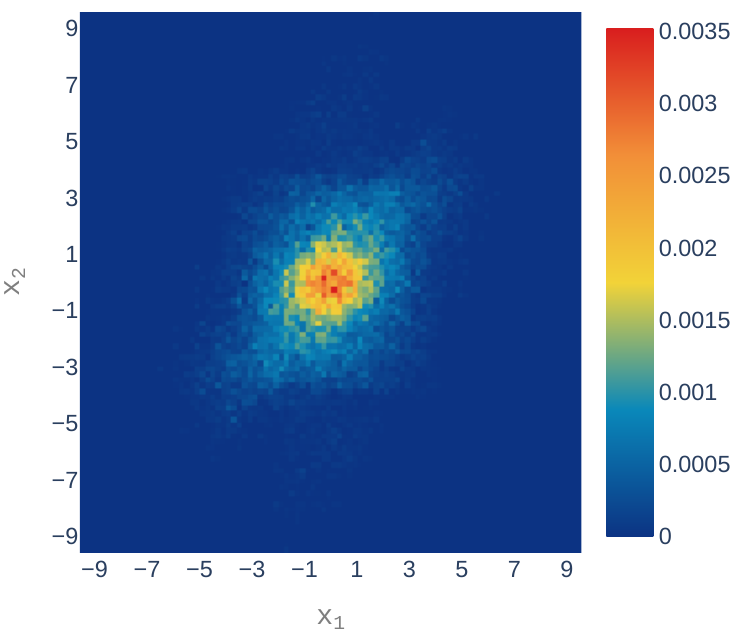}  
            \caption{Measure $\rho_3$}
            \label{Fig:joint_c}
        \end{subfigure} \\
        \begin{subfigure}[b]{0.5\textwidth}
            \centering
            \includegraphics[width=\textwidth]{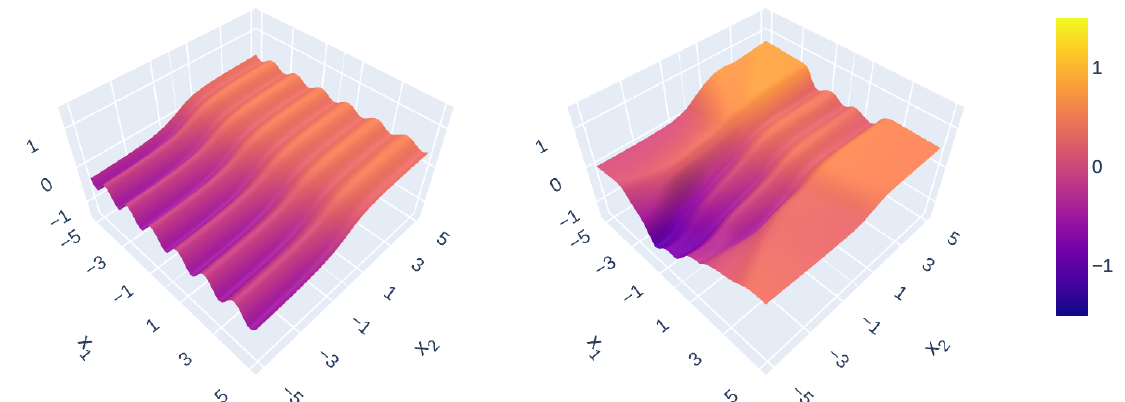}  
            \caption{True and estimated force potential $V$}
            \label{Fig:joint_d}
        \end{subfigure} &
        \begin{subfigure}[b]{0.21\textwidth}
            \centering
            \includegraphics[width=\textwidth]{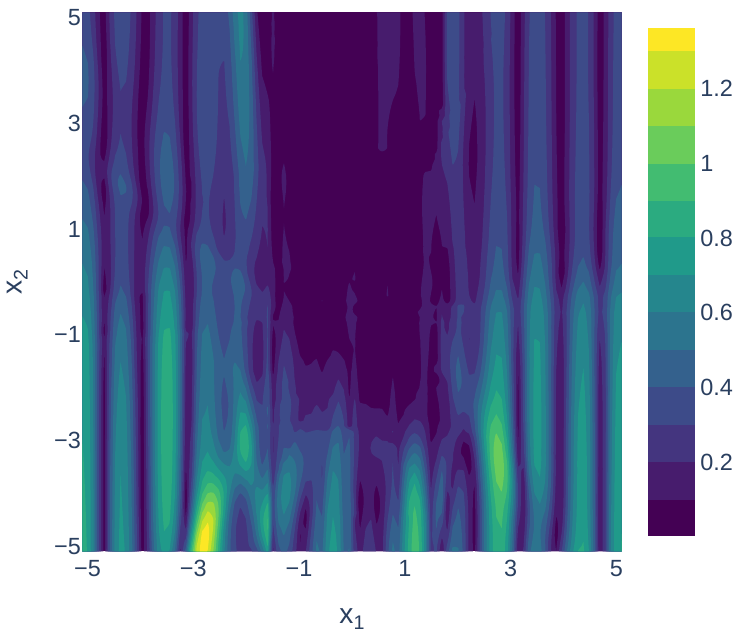}  
            \caption{$|\nabla V - \nabla \widehat{V}|$}
            \label{Fig:joint_e}
        \end{subfigure} &
        \begin{subfigure}[b]{0.21\textwidth}
            \centering
            \includegraphics[width=\textwidth]{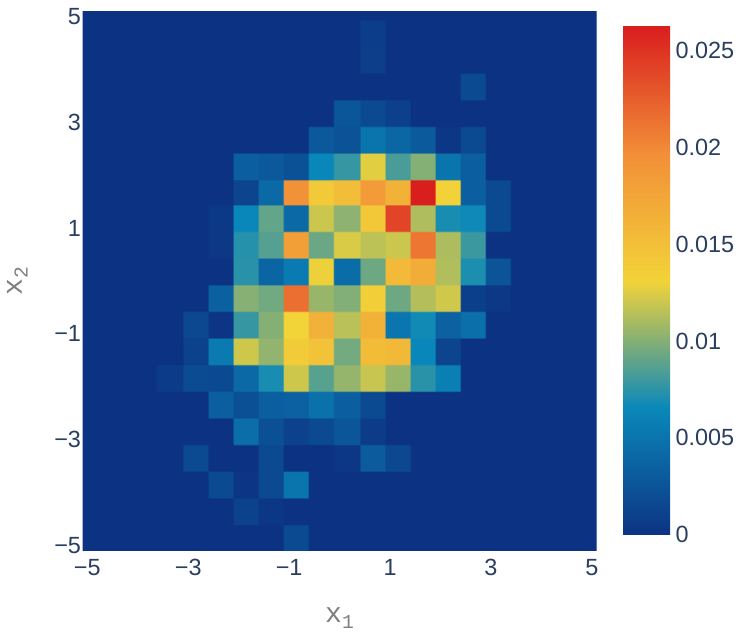}  
            \caption{Measure $\rho_2$}
            \label{Fig:joint_f}
        \end{subfigure}
    \end{tabular}\vspace{-3mm}
    \caption{Estimation result of the interaction and the force potential. The estimation results are accurate over the region where $\rho_2$ and $\rho_3$ are concentrated. }    \label{Fig:joint_all}\vspace{-3mm}
\end{figure}

Figure \ref{Fig:joint_all} presents the learned potentials. Figures \ref{Fig:joint_a} and \ref{Fig:joint_d} show the true and estimated interaction and force potentials, and the differences of their gradients are presented in Figures \ref{Fig:joint_b} and \ref{Fig:joint_e}. The estimators are accurate over the regions where data is concentrated, i.e., the large valued regions of the exploration measures, $\rho_2$ as in \eqref{eq:rho-V} for $V$ and $\rho_3$ as in \eqref{eq:rho-Phi} for $\Phi$, as shown in Figures \ref{Fig:joint_c} and \ref{Fig:joint_f}, respectively. These empirical measures are relatively rough since they are estimated from about $MNL = 4000$ and $MN^2L = 80000$ data samples for $\rho_2$ and $\rho_3$, respectively.  
 The final estimation error is $\| \nabla \Phi - \nabla \Phi^*\|_{L^2_{\rho_2}} = 0.5855$ and $\| \nabla V - \nabla V^*\|_{L^2_{\rho_2}} = 0.1746$.

To summarize, we overcome the challenge of unlabeled ensemble data without trajectory information by constructing a self-test loss function based on the weak-form equation of the empirical distributions. This self-test loss function is suitable for ensemble unlabeled data and neural network regression.

\begin{figure}[thb]
\centering
\includegraphics[width=0.7\textwidth]{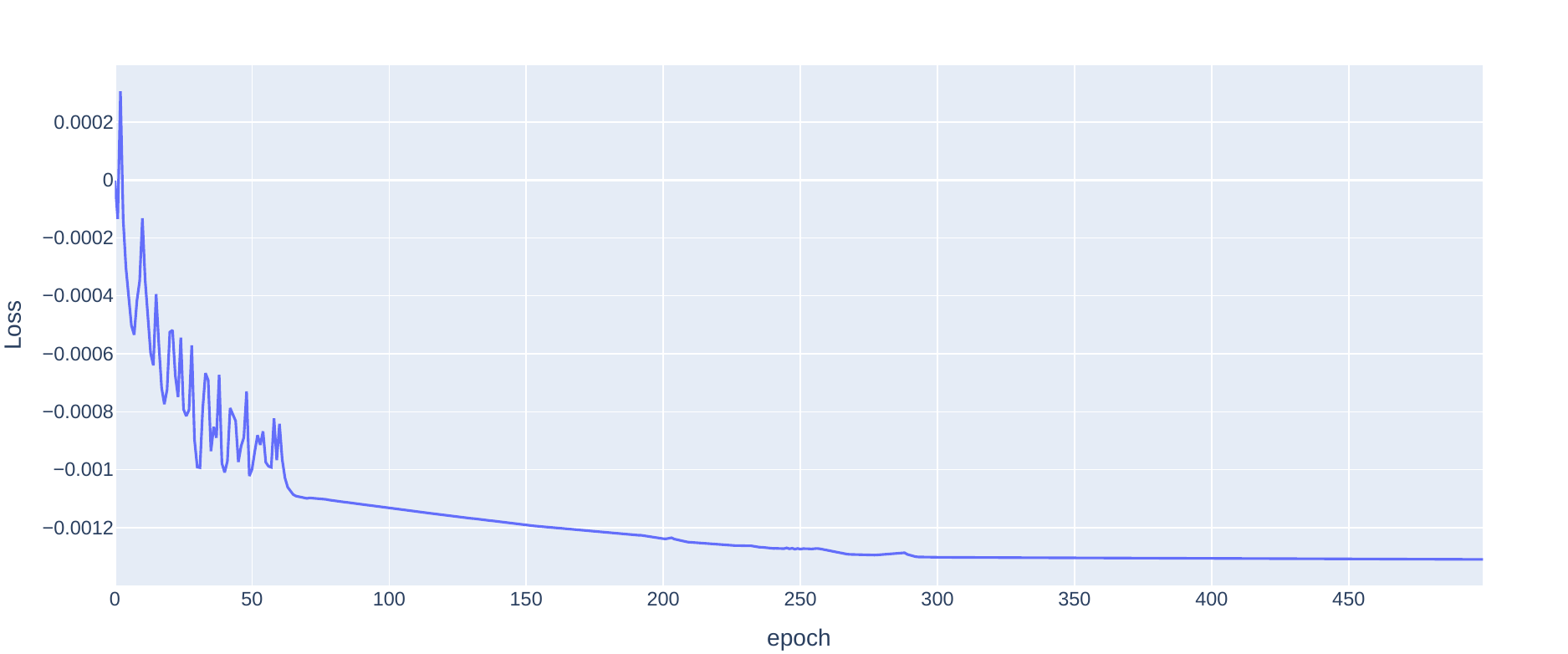}	
\caption{Training process. We utilized the \texttt{NAdam} optimizer and adjusted the learning rate when the loss plateaued. The initial oscillation is attributed to the reduction in the learning rate. Toward the end, the loss decreases more gradually because the learning rate is significantly lower than at the beginning. The final optimized loss value is -0.001309, corresponding to a normalized quadratic loss of $6.69\times 10^{-5}$.} 
\label{fig:training_process}\vspace{-3mm}
\end{figure}

\section{Conclusion}\label{sec:conclusion}

Discrete, noisy data pose substantial challenges for learning differential operators in PDEs and gradient flow systems. A standard approach is to construct loss functions based on weak-form equations, which avoids the large errors inherent in approximating high-order derivatives. However, this introduces the challenge of selecting suitable test functions.

This paper introduced a novel framework for constructing loss functions, called self-test loss functions. This method is designed for weak-form operators in PDEs and gradient flow systems. It applies to operators that depend linearly on the (function-valued) parameter to be estimated. By leveraging parameter—and data-dependent test functions, our approach automates the construction of loss functions and addresses the issue of test function selection.

The self-test loss function exhibits appealing theoretical and computational properties. It conserves energy in gradient flows and aligns with the expected log-likelihood ratio in stochastic differential equations. Furthermore, its quadratic structure enables a comprehensive analysis of the identifiability and well-posedness of the inverse problem. We demonstrate this by estimating the diffusion rate function, interaction potential, and kinetic potential in the aggregation-diffusion equation. Importantly, the self-test loss function supports the development of efficient parametric and nonparametric regression algorithms. Numerical experiments demonstrate that its minimizer is robust to noisy and discrete data, highlighting its practical utility and potential for broader applications.


 \appendix 

\section{Proofs and Derivations}\label{append:proofs}
\subsection{Proofs for Section \ref{sec:GF-energy}}\label{App:GF-liklihood}

\begin{proof}[Proof of Theorem \ref{thm:GF_EnergyConservation}] 
\textbf{Part (a)}.  Since $\phi_*$ is the true parameter, it satisfies the weak form of gradient flow $\pt_t u =- A_u \frac{\delta E_{\phi_*}}{\delta u}$. Applying a test function $\frac{\delta E_{\phi}}{\delta u}$ for any $\phi$ such that $E_\phi[u]<\infty$, we obtain
\[ 
\frac{dE_{\phi}[u]}{dt }= \innerp{\pt_t u, \frac{\delta E_{\phi}}{\delta u}} = -\innerp{A_u \frac{\delta E_{\phi_*}}{\delta u}, \frac{\delta E_{\phi}}{\delta u}}, \quad \forall t\in [0,T]. 
\] 
Integrating in time, we obtain
\[
E_\phi(u(T,\cdot)) - E_\phi(u(0,\cdot))= \int_0^T \frac{dE_{\phi}[u]}{dt }dt =  - \int_0^T \innerp{A_u \frac{\delta E_{\phi_*}}{\delta u}, \frac{\delta E_{\phi}}{\delta u}} \, dt. 
\]
Then, using the linearity of $\frac{\delta E_{\phi}}{\delta u}$ in $\phi$ due to Assumption \ref{assum:GF-abstract}, we write $\calE_{u_{[0,T]}}(\phi)$ in \eqref{eq:calE-cont-time} as  
  \begin{align}
   \calE_{u_{[0,T]}}(\phi) 
      & =-2 \int_0^T \innerp{A_u \frac{\delta E_{\phi_*}}{\delta u}, \frac{\delta E_{\phi}}{\delta u}} \, dt+  \int_0^T \la A_u \frac{\delta E_{\phi}}{\delta u}, \frac{\delta E_{\phi}}{\delta u} \ra dt  \notag \\
      & =   \int_0^T \la A_u \frac{\delta E_{\phi-\phi_*}}{\delta u}, \frac{\delta E_{\phi-\phi_*}}{\delta u} \ra dt  - \int_0^T \innerp{A_u \frac{\delta E_{\phi_*}}{\delta u}, \frac{\delta E_{\phi_*}}{\delta u}} \, dt. \label{eq:E-Au-square} 
  \end{align}
  From $\innerp{A_u\xi,\xi}\geq 0$ in \eqref{eq:Au-properties}, the first term is non-negative. Thus, 
       we know  that  
       $$\calE_{u_{[0,T]}}(\phi)\geq \calE_{u_{[0,T]}}(\phi_*)=-\int_0^T\innerp{A_u \frac{\delta E_{\phi_*}}{\delta u}, \frac{\delta E_{\phi_*}}{\delta u}} \ud t$$ and $\phi_*$   is a minimizer of $\calE_{u_{[0,T]}}(\phi)$.  

\textbf{Part (b)}. It follows from \eqref{eq:E-Au-square} that $\phi_*$ is the unique minimizer in $\calH$ if \eqref{eq:Au-positive} holds.

\textbf{Part (c)}.    
Since $\phi_0$ is a minimizer of $\calE_{u_{[0,T]}}(\phi)$ and $E_\phi$ is linear in $\phi$, we have, for any $\psi$,  
\begin{equation*}
\begin{aligned}
	0& = \frac{d}{d\epsilon } \calE_{u_{[0,T]}}(\phi_0+\epsilon \psi) = \lim_{\epsilon\to 0} \frac{\calE_{u_{[0,T]}}(\phi_0+\epsilon \psi)- \calE_{u_{[0,T]}}(\phi_0)}{\epsilon} \\
& = 2[E_\psi(u(T,\cdot)) - E_\psi(u(0,\cdot))]+ 2\int_0^T \innerp{A_u \frac{\delta E_{\phi_0}}{\delta u}, \frac{\delta E_{\psi}}{\delta u}} \, dt. 
\end{aligned}
\end{equation*}
Taking $\psi=\phi_0$, we obtain \eqref{eq:enerby-conservation}. 
\end{proof}

\vspace{3mm}

\begin{proof}[Proof of Theorem \ref{thm:stloss-likelihood}]

The Fokker-Planck equation of the Mckean-Vlasov SDE is \eqref{eq:FPE-V-Phi}. The self-test loss function for estimating $(V,\Phi)$ using its weak form is given in \eqref{eq:stLoss-mf}, which reads
\[
\calE_{u_{[0,T]}}(\Phi,V)   : =  \frac{1}{T}\int_0^T \int_{\R^d} \left[ u | \nabla \Phi * u + \nabla V |^2 - 2 (\partial_t u-   \nu  \Delta u) (\Phi*u + V )  \right ] dx\ dt. 
\]
On the other hand, by Girsanov Theorem (see e.g.,\cite{oksendal2013_sde}), the negative log-likelihood ratio for $\barX_{[0,T]}$ is 
\begin{equation*}
 \mathcal{E}_{\barX_{[0,T]}}(\phi)= - \ln \frac{d \P_\phi }{d\P_0} =  \frac{1}{2\nu}\int_{0}^T \left(\big| [\nabla \Phi *u +\nabla V](\barX_t)\big|^2 dt - 2 \big \langle [\nabla \Phi*u + \nabla V] (\barX_t), d\barX_t \big \rangle \right),   
\end{equation*}
where $\P_\phi$ and $\P_0$ are the distributions of the path under the SDE with parameters $\phi=(V,\Phi)$ and $V=\Phi=0$, respectively. 
Taking expectation and using the fact that $\barX_t \sim u(\cdot,t)$, 
\begin{align*}
\E \big[ \mathcal{E}_{\barX_{[0,T]}}(\phi) \big]  =  \frac{1}{2\nu}\int_0^T \int_{\R^d}  \big[ \big| \nabla \Phi *u +\nabla V \big|^2  u \,dx - 2\E\big[ \big \langle [\nabla \Phi*u + \nabla V] (\barX_t), d\barX_t \big \rangle \big ]  \ dt. 
\end{align*}
To compute the above expectation, using $d\barX_t$ from the SDE with the fact that the martingale term has expectation 0 and applying integration by parts, we have
\begin{align*}
& \, \E[ \big \langle [\nabla \Phi*u + \nabla V] (\barX_t), d\barX_t \big \rangle ] = \E[ \big \langle \nabla [ \Phi*u +  V] (\barX_t), -\nabla [V_{*}+ \Phi_{*} *u](\barX_t) ] \big \rangle ] \\
= & \, \int_{\R^d}  \big \langle  \nabla [\Phi*u +  V], -u \nabla [V_{*}+ \Phi_{*}*u ]\big \rangle dx  \\
= & \, \int_{\R^d}  (\Phi*u +  V)   \nabla\cdot \big[ u \nabla (V_{*}+ \Phi_{*}*u ) \big ] dx = \int_{\R^d}  (\Phi*u +  V)   (\partial_t u -\nu \Delta u) \big ] dx, 
\end{align*}
where the last equation follows from the Fokker-Planck equation \eqref{eq:FPE-V-Phi} with parameters $(V_{*},\Phi_{*})$. Combining the above two equations, we have $ \calE_{u_{[0,T]}}(\Phi,V)= \mathcal{E}_{\barX_{[0,T]}}(\phi)$. 
\end{proof}

\subsection{Proofs for Section \ref{sec:ID}}\label{appd:ID}

\begin{proof}[Proof of Proposition \ref{prop:est_h}]
Recall that in \eqref{eq:R-h''}, $\widetilde f = -\nabla \cdot \big[ u\nabla [\nu h_*'(u) + (\Phi_*-\Phi)*u + V_*-V] \big]=: - \nabla\cdot [ u \nabla F[u]]$, where we set $F[u]=\nu h_*'(u) + (\Phi_*-\Phi)*u + V_*-V $. Then, 
\begin{align*}
	\innerp{\widetilde{f},\, v_{h}[u]} & = \innerp{-\nabla \cdot [ u \nabla F[u]], h'(u)} 
= \int_{\R^d}  u(x) \nabla F[u](x)\cdot \nabla u(x) h''(u(x) )\,  dx  \\
& \leq \big( \int_{\R^d}  u(x) |\nabla F[u](x)|^2 | dx \big) ^{1/2} \big( \int_{\R^d}  u(x) |\nabla u(x)|^2 h''(u(x) )^2\,  dx \big) ^{1/2}<+\8. 
\end{align*} 
Thus, the Riesz representation theorem gives a data-dependent $h_{\mathcal{D}}\in L^2_{\rho_1}$ with $\rho_1$ defined in \eqref{eq:rho-h} such that  
\begin{align*}
	\sum_{l=1}^L  \innerp{\widetilde{f}_l,\, v_{h}[u_l]} & =\sum_{l=1}^L \innerp{-\nabla \cdot [ u_l \nabla F[u_l]], h'(u_l)} \\
& = \sum_{l=1}^L \int_{\R^d}  u_l \nabla F[u_l] \cdot \nabla u_l h''(u_l(x) )\,  dx  
 = :  \innerp{h_{\mathcal{D}}, h''}_{L^2_{\rho_1}}. 
\end{align*} 
Then, we can write the self-test loss function as
\begin{align*}
\calE_1(h'') & = \sum_{l=1}^L \innerp{R_h[u_l]-2\widetilde{f}_l,\, v_h[u_l]} + C_0 
	     = \|h''\|_{L^2_{\rho_1}}^2 - 2 \innerp{h'',h_{_{\mathcal{D}}}}_{L^2_{\rho_1}} + C_0. 
	\end{align*}
 The Fr\'echet derivative of $\calE_1$ in terms of the variable $h''$ is $D_{h''} \calE_1(h'') = 2h'' -2 h_{_{\mathcal{D}}}$. Thus, the minimizer of $\calE_1$ is unique and
\begin{equation*}
	\widehat{h''} = \argmin{h''\in L^2_{\rho_1}} \calE_1(h'') = I^{-1} h_{_{\mathcal{D}}}
\end{equation*} 
with $I$ being the identity operator on $L^2_{\rho_1}$. Thus, this inverse problem is well-posed. 
\end{proof} 

\vspace{3mm}

\begin{proof}[Proof of Proposition \ref{prop:est-V}]
First, recall that in \eqref{eq:R-V}, $\widetilde f = -\nabla \cdot \big[ u\nabla [\nu h_*'(u)-\nu h'(u) + (\Phi_*-\Phi)*u + V_*] \big]:= - \nabla\cdot [ u \nabla F[u]]$, where we set $F[u]=\nu h_*'(u) - \nu h'_*(u) + (\Phi_*-\Phi)*u + V_* $.
Thus, the linear term in the loss function is   
\begin{align*}
	\sum_{l=1}^L  \innerp{\widetilde{f}_l,\, v_{V}[u_l]} & =\sum_{l=1}^L \innerp{-\nabla \cdot [ u_l \nabla F[u_l]], V} 
  = \sum_{l=1}^L \innerp{ u_l \nabla F[u_l] \cdot \nabla V }   
 =:   \innerp{\overrightarrow{V_{\mathcal{D}}}, \nabla V}_{L^2_{\rho_2}}, 
\end{align*} 
where  $\overrightarrow{V_{\mathcal{D}}} \in L^2_{\rho_2}(\bR^d; \bR^d)$ with $\rho_2$ defined in \eqref{eq:rho-V} by the Riesz representation theorem. In particular, we have $\overrightarrow{V_{\mathcal{D}}} = \nabla F[u_l]$ when $F[u_l]$ is independent of $l$ (e.g., when $L=1$).  
Then, we can write the self-test loss function as
\begin{align*}
\calE_2(\nabla V) & = \sum_{l=1}^L \innerp{R_V[u_l]-2\widetilde{f}_l,\, V}  
        = \|\nabla V\|_{L^2_{\rho_2}}^2 - 2 \innerp{\nabla V,\overrightarrow{V_{\mathcal{D}}} }_{L^2_{\rho_2}} + C_0. 
	\end{align*}
 Regarding $\calE_2(\nabla V)$ as a functional of $\mathbf{v}=\nabla V\in L^2_{\rho_2}(\R^d; \R^d)$, we define
 $
 \calE_2(\mathbf{v}) = \|\mathbf{v}\|_{L^2_{\rho_2}}^2 - 2 \innerp{\mathbf{v},\overrightarrow{V_{\mathcal{D}}} }_{L^2_{\rho_2}} + C_0.$
The Fr\'echet derivative of $\calE_2$ 
over $L^2_{\rho_2}(\R^d; \R^d)$ is 
$D_{\mathbf{v}} \calE_2(\mathbf{v}) =2(\mathbf{v}-\overrightarrow{V_{\mathcal{D}}}).$
Thus, the minimizer of $\calE_2$, denoted as $\widehat{\nabla V}$, is unique and satisfies
\begin{align*}
\widehat{\nabla V} = \argmin{ \nabla V\in L^2_{\rho_2}(\R^d;\R^d) } \calE_2(\nabla V) = I^{-1} \overrightarrow{V_{_\mathcal{D}}}, 
\end{align*}
where $I$ is the identity operator in $L^2_{\rho_2}(\R^d;\R^d)$, and the inverse problem is well-posed. 

To identify $V$, we regard the self-loss function as a functional from $\calH_0$ to $\R$: 
\[ \widetilde\calE_2(V):= \calE_2(\nabla V)= \|\nabla V\|_{L^2_{\rho_2}}^2 - 2 \innerp{\sum_{l=1}^L \widetilde{f}_l,V}_{L^2} + C_0 . 
\] 
 Using Poincare inequality \eqref{eq:Poincare}, we have 
$   \int_{\R^d} |V|^2 \rho_2   \ud x  \leq c \int_{\R^d} |\nabla V|^2 \rho_2 \ud x  
$, 
where $c>0$ the Poincare constant.  
This implies 
$\|V\|_{\calH_0}^2:=\| V\|_{H^1_{\rho_2}}^2 \leq (1+c)\|\nabla V\|_{L^2_{\rho_2}}^2.$
Combining this with H\"older's inequality for
 \begin{align*}
  |\innerp{\sum_{l=1}^L \widetilde{f}_l,V}_{L^2_{\rho_2}}| =   |\innerp{\nabla V,\overrightarrow{V_{_\mathcal{D}}} }_{L^2_{\rho_2}}| \leq  \frac1 {4(1+c) } \|\nabla V\|_{L^2_{\rho_2}}^2 + 4(1+c) \| \vec{V}_{_\mathcal{D}}\|_{L^2_{\rho_2}}^2,
 \end{align*}
so we have  
\[ \widetilde\calE_2(V)  \geq \frac1 {2(1+c) } \|V\|_{H^1_{\rho_2}}^2 +C_0 -  8(1+c) \| \vec{V}_{_\mathcal{D}}\|_{L^2_{\rho_2}}^2. 
\] 
Hence, the functional $\widetilde\calE_2(V)$ is uniformly convex on $\calH_0$, and it has a unique minimizer in $\calH_0$. 

If $\widehat{V}$ minimizes $\widetilde\calE_2(V)$, the first variation (Gateaux derivative) of $\widetilde\calE_2(V)$ is
\begin{align*}
    \frac{\ud}{\ud \eps}\Big|_{\eps=0} \int |\nabla(\widehat{V} +\eps \tilde{V})|^2 \rho_2 - 2(\widehat{V} +\eps \tilde{V})\sum_{l=1}^L \widetilde{f}_l \,\ud x 
    =2 \int (\nabla \widehat{V} \nabla \tilde{V}\rho_2 - \sum_{l=1}^L \widetilde{f}_l \, \tilde{V} ) \ud x=0
\end{align*}
for any $\tilde{V}\in \calH_0$. Hence, the minimizer $\widehat{V}$ satisfies \eqref{eq:Vhat}.
\end{proof}

\vspace{3mm}

 \begin{proof}[Proof of Proposition \ref{prop:est-Phi}] 
 First, write the quadratic term in the loss function \eqref{eq:stloss-Phi} as 
\begin{align*}
& \sum_{l=1}^L  \int_{\R^d} u_l(x)|\nabla \Phi*u_l(x) |^2dx \\ 
=& \sum_{l=1}^L \int u_l(x)  \int \nabla \Phi(y) u_l(x-y) dy \cdot \int \nabla \Phi(y') u_l(x-y') dy'  dx \\
= &\int \int \innerp{ \nabla \Phi(y),\, \nabla \Phi(y')}_{\R^d}  \big[  \sum_{l=1}^L \int  u_l(x)   u_l(x-y) u_l(x-y')  dx \big] dy  dy' 
=  \innerp{\nabla \Phi,    L_\Gbar\nabla\Phi}_{L^2_{\rho_3}}, 
\end{align*}
with $L_\Gbar$ defined in \eqref{eq:L_Gbar} and $\rho_3$ defined in \eqref{eq:rho-Phi}.

Second, the Riesz representation theorem gives a vector-valued function $\overrightarrow\Phi_{_\mathcal{D}}:\R^d\to \R^d$ such that the linear term in the loss function can be written as 
\begin{align*}
	 \sum_{l=1}^L \int_{\R^d}  u_l(x)    \nabla F[u_l](x) \cdot \nabla \Phi*u_l(x) dx & 
	 = \innerp{\overrightarrow\Phi_{_\mathcal{D}}, \nabla \Phi}_{L^2_{\rho_3}}.  
\end{align*}

Then, we can write the loss function in \eqref{eq:stloss-Phi} as 
\begin{equation}
	\calE_3(\nabla\Phi)  = \innerp{\nabla \Phi,    L_\Gbar\nabla\Phi}_{L^2_{\rho_3}} - 2 \innerp{\overrightarrow\Phi_{_\mathcal{D}}, \nabla \Phi}_{L^2_{\rho_3}} + C_0. 
\end{equation}
Regarding $\calE_3$ as a functional in terms of $\nabla \Phi$, the Fr\'echet derivative of $\calE_3$ is $D_{\nabla\Phi} \calE_3(\nabla \Phi) = 2L_\Gbar \nabla\Phi -2 \overrightarrow\Phi_{_{\mathcal{D}}}$. Thus, the minimizer of $\calE_3$ is unique in $\mathrm{Null}(L_\Gbar)^\perp$ and
\begin{equation}\label{eq:Phi-hat-argmin}
	\widehat{\nabla \Phi} = \argmin{\nabla \Phi \in \mathrm{Null}(L_\Gbar)^\perp\subset L^2_{\rho_3}} \calE_3(\nabla \Phi) = L_\Gbar^{-1} \overrightarrow\Phi_{_{\mathcal{D}}}, 
\end{equation} 
where  $L_\Gbar^{-1}$ is the pseudo-inverse of the operator $L_\Gbar$. Since the operator $L_\Gbar$ is compact, so $\mathrm{Null}(L_\Gbar)\neq \{0\}$ and the above inverse problem is ill-posed.
\end{proof}

\vspace{3mm}

\begin{proof}[Proof of Proposition \ref{prop:est-joint}]
We solve for the estimators by setting the Fr\'echet derivatives of the loss function to zero. We write the loss function in \eqref{eq:stloss-jiont} as  
\begin{equation*}
\begin{aligned}
    \calE(h'',\nabla V,\nabla \Phi)   
  &   =  \|h''\|_{L^2_{\rho_1}}^2 + \|\nabla V\|_{L^2_{\rho_2}}^2+ \innerp{\nabla \Phi,    L_\Gbar\nabla\Phi}_{L^2_{\rho_3}}  \\ 
       + & 2\sum_{l=1}^L   \int_{\R^d} u_l \nu h''(u_l) \nabla u\cdot (\nabla V+ \nabla \Phi*u) dx 
         +   2\sum_{l=1}^L   \int_{\R^d} u_l \nabla V\cdot \nabla \Phi*u_l dx \\
       - &  2\sum_{l=1}^L   \int_{\R^d} u_l \nabla v_{\phi_*}[u_l] \cdot \nabla [ \nu h'(u_l) +\Phi*u_l + V ] dx. 
\end{aligned} 
\end{equation*}
Recall   $\rho_1,\rho_2, \rho_3$  defined in \eqref{eq:rho-h},\eqref{eq:rho-V} and \eqref{eq:rho-Phi}, respectively. We have that 
\begin{align*}
\innerp{ D_{h''} \calE(h'',\nabla V,\nabla \Phi), g_1}_{L^2_{\rho_1}}  & =  \innerp{2(h'' + M_{hV}\nabla V + M_{h\Phi}\nabla \Phi- h_{\mathcal{D}},g_1  }_{L^2_{\rho_1}}, \,
	\\
\innerp{ D_{\nabla V} \calE(h'',\nabla V,\nabla \Phi), \vec{g}_2}_{L^2_{\rho_2}}  & =  \innerp{2(M_{Vh} h'' + \nabla V + M_{V\Phi}\nabla \Phi- \overrightarrow{V_{\mathcal{D}}},\vec{g}_2 }_{L^2_{\rho_2}},  \,  
 \\
\innerp{ D_{\nabla \Phi} \calE(h'',\nabla V,\nabla \Phi), \vec{g}_3}_{L^2_{\rho_3}}  & =  \innerp{2(M_{\Phi h} h'' + M_{\Phi V} \nabla V + L_\Gbar \nabla \Phi- \overrightarrow{\Phi_{\mathcal{D}}},\vec{g}_3 }_{L^2_{\rho_2}},  \, 
\end{align*}
$\forall  g_1\in L^2_{\rho_1},\, \vec{g}_2 \in L^2_{\rho_2} $, and $\vec{g}_3 \in L^2_{\rho_3}$. Here, the operators $M_{ab}$ are defined from the cross-product terms in the loss function. For example, 
\begin{align*}
	\innerp{M_{hV} \nabla V,g_1}_{L^2_{\rho_1}} & = \sum_{l=1}^L   \int_{\R^d} u_l \nu g_1 \nabla u\cdot \nabla V dx; \\
  \innerp{M_{Vh} h'',\vec{g}_2}_{L^2_{\rho_2}} &= \sum_{l=1}^L   \int_{\R^d} u_l \nu h''\nabla u \cdot \vec{g}_2 ; \\
  \innerp{M_{ V \Phi} \nabla \Phi,\vec{g}_2}_{L^2_{\rho_2}}&=\sum_{l=1}^L   \int_{\R^d} u_l \vec{g}_2 \cdot \nabla \Phi*u_l dx;\\
	\innerp{M_{\Phi V} \nabla V, \vec{g}_3}_{L^2_{\rho_3}} & = \sum_{l=1}^L   \int_{\R^d} u_l \nabla V \cdot \vec{g}_3*u_l dx . 
\end{align*} 
In particular, since 
\begin{align*}
    &\innerp{M_{hV} \vec{b},g_1}_{L^2_{\rho_1}} = \innerp{M_{Vh} g_1, \vec{b}}_{L^2_{\rho_2}}, \quad \forall \vec{b}\in L^2_{\rho_2}, \, g_1\in L^2_{\rho_1},\\
    &\innerp{M_{\Phi V} \vec{g}_2, \vec{g}_3}_{L^2_{\rho_3}} = \innerp{M_{ V \Phi} \vec{g}_3,\vec{g}_2}_{L^2_{\rho_2}}, \quad \forall \vec{g}_2\in L^2_{\rho_2}, \, \vec{g}_3 \in L^2_{\rho_3},
\end{align*}
we have joint operators $M_{hV}^*= M_{Vh}$ and $M_{\Phi V}^*= M_{V \Phi}$ with operator norms satisfying $\|M_{hV}\|\leq 1 $ and $\|M_{\Phi V}\|\leq \|L_\Gbar\|^{1/2}$. 
Then, the joint estimator solves the system 
\begin{equation}\label{jointA}
	\begin{pmatrix}
		I_{L^2_{\rho_1}} & M_{hV} & M_{h\Phi}  \\
		M_{Vh}  & I_{L^2_{\rho_2}} &  M_{V\Phi}  \\
		M_{\Phi h}  &  M_{\Phi V} & L_\Gbar \\
	\end{pmatrix} 
	\begin{pmatrix}
		h''\\
		\nabla V \\
		\nabla \Phi 
	\end{pmatrix}
	 = 	\begin{pmatrix}
		h_{\mathcal{D}}\\
		\overrightarrow{V_{\mathcal{D}}} \\
		\overrightarrow{\Phi_{\mathcal{D}}}
	\end{pmatrix}. 
\end{equation}

The Hessian (the second variation) of the loss function is the operator on the left-hand-side of \eqref{jointA}, and denote it by $A: L^2_{\rho_1}(\R^+)\otimes L^2_{\rho_2}(\R^d)\otimes L^2_{\rho_3}(\R^d)\to L^2_{\rho_1}(\R^+)\otimes L^2_{\rho_2}(\R^d)\otimes L^2_{\rho_3}(\R^d)$. The operator $A$ is self-adjoint and semi-positive definite.  

We show first that $\phi=(0,\mathbf{c},-\mathbf{c})$ with a nonzero $\mathbf{c}\in \R^d$ is an eigenfunction of $A$ corresponding to the zero eigenvalue. 
Note that by definition, $\innerp{M_{V\Phi} \mathbf{c},\mathbf{c}}_{L^2_{\rho_2}}=\sum_{l=1}^L   \int_{\R^d} u_l  \mathbf{c} \cdot  \mathbf{c}*u_l dx = \|\mathbf{c}\|^2 $ and similarly, $\innerp{M_{\Phi V} \mathbf{c},\mathbf{c}}_{L^2_{\rho_3}} =\|\mathbf{c}\|^2 $. Meanwhile, we have $ L_\Gbar \mathbf{c} = \int \int  \frac{G(y,y')}{\dot\rho_3(y)} \mathbf{c} d y' \, dy =  \mathbf{c}$ by the definition of $G$.   
It follows that 
\begin{align*}   
\la A\phi, \phi\ra_{L^2_{\rho_1}\otimes L^2_{\rho_2}\otimes L^2_{\rho_2}} =& \la \begin{pmatrix}
		M_{hV} \mathbf{c} - M_{h\Phi} \mathbf{c}\\
		I_{L^2_{\rho_2}} \mathbf{c}-  M_{V\Phi} \mathbf{c} \\
		M_{\Phi V} \mathbf{c} - L_\Gbar \mathbf{c} 
	\end{pmatrix}, \begin{pmatrix}
		0\\
		\mathbf{c} \\
		-\mathbf{c}  
	\end{pmatrix} \ra_{L^2_{\rho_1}\otimes L^2_{\rho_2}\otimes L^2_{\rho_2}}\\
    =& \la I_{L^2_{\rho_2}} \mathbf{c} -  M_{V\Phi} \mathbf{c}, \mathbf{c} \ra_{L^2_{\rho_2}} -  \la M_{\Phi V} \mathbf{c} - L_\Gbar \mathbf{c}, \mathbf{c}\ra_{L^2_{\rho_3}}=0. 
    \end{align*} 

Lastly, note that for $\phi_n=(0,0,\psi_n  )$, where $\psi_n$ is an eigenfunction of $L_\Gbar$ such that $L_\Gbar\psi_n= \lambda_n\psi_n$, we have $\innerp{A\phi_n,\phi_n} = \lambda_n $, where $\lambda_n \to 0$ as $n\to\infty$ since $L_\Gbar$ is compact. Thus, the loss function is not uniformly convex, and the joint estimation is ill-posed.
\end{proof}

\subsection{Derivation details for Section \ref{sec:num-Phi}}\label{append:num-phi-details}
\begin{proof}[Derivation of Eq.\eqref{eq:loss_phi_radial}]
Using the facts that $\nabla \Phi(|x|) = \phi(|x|)\frac{x}{|x|} $ and 
$$\nabla \Phi * u (x)= \int_\R \phi(|y|)\frac{y}{|y|} u(x-y)dy = \int_0^\infty \phi(r)[ u(x-r)- u(x+r)] dr, 
 $$
 along with the notation $\delta u(x,r;t)$ in \eqref{eq:rho_radial}, we can write the integrals as 
\begin{align*}
 \int_{\R}  u | \nabla \Phi * u|^2 dx dt 
= &  \int_{0}^\infty  \int_{0}^\infty \phi(r)\phi(s) \int_\R u(x)\delta u(x,r)  \delta u(x,s)dx drds dt, \\
= &  \int_{0}^\infty  \int_{0}^\infty \phi(r)\phi(s) G(r,s) drds = \int_{0}^\infty  \int_{0}^\infty \phi(r)\phi(s) \Gbar(r,s) \dot\rho(r)\dot\rho(s)drds. 
\end{align*}
where the integral kernels $G, \Gbar:\R_+\times \R_+\to \R$ are defined in \eqref{eq:GGbar_radial}. 

Denote $F(x) = \int_{0}^x f(y)dy$. Integration by parts with $\Phi* u(10)= \Phi* u(0) =0$ implies that
\begin{align*}
	 \int_{\R} f(x)  \Phi*u(x) dx & =   F(x) \Phi* u(x)\big | _0^{10} -  \int_0^2  \phi(r) \int_0^{10} F(x)[u(x-r) - u(x+r)]\,dx\,dr \\ 	 & =    \int_0^2  \phi(r) F_{f,u}(r)\,dr,  
\end{align*} 
where $F_{f,u}(r) := - \int_0^{10} F(x)[u(x-r) - u(x+r)]\,dx$. 
 Combining the above equations, we obtain \eqref{eq:loss_phi_radial}. 
\end{proof}

\ifarXiv
\section*{Acknowledgment}
Yuan Gao was partially supported by NSF under Award DMS-2204288. Fei Lu was partially supported by NSF DMS-2238486 and DMS-2511283. 
\fi

{\small 
\bibliographystyle{plain} 
\bibliography{ref_FeiLU2024_11,ref_gao,ref_IPS_learning2404,ref_weakPDE_learning2411,ref_quanjun}
}
\end{document}